\documentclass[mnsc]{informs3}

\OneAndAHalfSpacedXII % Current default line spacing
\usepackage{natbib}
 \bibpunct[, ]{(}{)}{,}{a}{}{,}%
 \def\bibfont{\small}%

\usepackage[pdfencoding=auto]{hyperref}
\usepackage{bookmark}
 
\usepackage[ruled]{algorithm2e} % For algorithms
\usepackage{algorithmic}
\usepackage{enumitem}
\usepackage{bbm}

\usepackage{comment}

\TheoremsNumberedThrough     % Preferred (Theorem 1, Lemma 1, Theorem 2)
\EquationsNumberedThrough    % Default: (1), (2), ...
\newcommand{\vucb}{v^{\sf UCB}}
\newcommand{\vlcb}{v^{\sf LCB}}

\newcommand{\prob}{\mathbb{P}}

\newcommand{\sumt}{\sum_{t = 1}^T}

\newcommand{\sumS}{\sum_{S \in \mathcal{S}}}
\newcommand{\epochl}{\mathcal{E}_\ell}

\newcommand\numleq[1]%
  {\stackrel{\scriptscriptstyle(\mkern-1.5mu#1\mkern-1.5mu)}{\leq}}
\newcommand{\vmax}{v_{\max}}

\usepackage{color-edits}
\addauthor{vg}{magenta}
\addauthor{aa}{blue}
\addauthor{np}{red}

\begin{document}
%%%%%%%%%%%%%%%%

%\setcounter{page}{1}
%\DOI{10.1287/xxxx.0000.0000}%

 %\RUNAUTHOR{Jones}
 \RUNAUTHOR{Aznag, Goyal, and Perivier}

\RUNTITLE{MNL-Bandit with Knapsacks}
%\RUNTITLE{}

\TITLE{MNL-Bandit with Knapsacks: a near-optimal algorithm}
%\TITLE{}

\ARTICLEAUTHORS{%
\AUTHOR{Abdellah Aznag\qquad\qquad Vineet Goyal\qquad\qquad Noemie Perivier}
\AFF{Department of Industrial Engineering and Operations Research, Columbia University}
} % end of the block

\ABSTRACT{%
We consider a dynamic assortment selection problem where a seller has a fixed inventory of $N$ substitutable products and faces an unknown demand that arrives sequentially over $T$ periods. In each period, the seller needs to decide on the assortment of products (satisfying certain constraints) to offer to the customers. The customer's response follows an unknown multinomial logit model (MNL) with parameter $\boldsymbol{v}$. If customer selects product $i \in [N]$, the seller receives revenue $r_i$. The goal of the seller is to maximize the total expected revenue from the $T$ customers given the fixed initial inventory of $N$ products. We present MNLwK-UCB, a UCB-based algorithm and characterize its regret under different regimes of inventory size. We show that when the inventory size grows quasi-linearly in time, MNLwK-UCB achieves a $\Tilde{O}(N + \sqrt{NT})$ regret bound. We also show that for a smaller inventory (with growth $\sim T^{\alpha}$, $\alpha < 1$), MNLwK-UCB achieves a $\Tilde{O}(N(1 + T^{\frac{1 - \alpha}{2}}) + \sqrt{NT})$. In particular, over a long time horizon $T$, the rate $\Tilde{O}(\sqrt{NT})$ is always achieved regardless of the constraints and the size of the inventory.
}%

%\KEYWORDS{deterministic inventory theory; infinite linear programming duality; 
%  existence of optimal policies; semi-Markov decision process; cyclic schedule}


%\vspace{-3cm}
\maketitle
%%%%%%%%%%%%%%%%%%%%%%%%%%%%%%%%%%%%%%%%%%%%%%%%%%%%%%%%%%%%%%%%%%%%%%

% Samples of sectioning (and labeling) in MSOM
% NOTE: (1) \section and \subsection do NOT end with a period
%       (2) \subsubsection and lower need end punctuation
%       (3) capitalization is as shown (title style).
%
%\section{Introduction.}\label{intro} %%1.
%\subsection{Duality and the Classical EOQ Problem.}\label{class-EOQ} %% 1.1.
%\subsection{Outline.}\label{outline1} %% 1.2.
%\subsubsection{Cyclic Schedules for the General Deterministic SMDP.}
%  \label{cyclic-schedules} %% 1.2.1
%\section{Problem Description.}\label{problemdescription} %% 2.

\section{Introduction}
Modeling customer choice has been a central challenge in revenue management. One of the important attributes to capture is the substitution effect, where the presence or absence of one product affects the probability of sale of another product. Choice models attempt to do so by specifying the probabilities of sale of a given product within an offered subset. In practice, this paradigm comes with two challenges. The first challenge is the uncertainty aspect of the customer behavior. The choice function is unknown and the seller needs to simultaneously learn the choice model while maximizing their revenue. This gives rise to many sequential decision making strategies that attempt to solve the underlying \text{exploration-exploitation} trade-off. The second challenge arises from  finite initial inventory. The seller has to factor-in the specifics of the choice model in their inventory management. This leads to knapsack-like considerations that attempt to solve the underlying \text{capacity-demand} trade-off. While both challenges have been successfully addressed separately for many known choice models, their joint interaction presents new challenges that are not fully understood. It is precisely this joint interaction that we study in this paper. The Multinomial Logit Model (MNL) is the most widely used choice model for assortment selection problems, for both its simplicity and tractability. It is built around the assumption that each product has an intrinsic \textit{utility} that is proportional to the probability of being chosen within an offered assortment. Even for such a fundamental model, the question of how a finite inventory should be managed when the utilities are unknown is still not satisfactorily answered. 

We study a dynamic assortment optimization problem under bandit feedback, where a seller with a fixed initial inventory of $N$ substitutable products faces a sequence of i.i.d. customer arrivals (with an unknown distribution) over a time horizon of $T$ periods, and needs to decide in each period on an assortment of products to offer to the customer to maximize the total expected revenue. Such a problem arises in many applications including online retail, advertisement, and recommendation systems. 
"The seller initially has no (or only limited) information about the preferences of the customers and needs to learn them through repeated interaction with \text{i.i.d.} customers. Specifically, in each period, the seller offers an assortment to the customer; the customer makes a choice from the assortment according to the unknown preferences of the MNL model, and the seller only observes the eventual choice from the given assortment and needs to update the estimate and future actions under this bandit feedback. This problem exemplifies the classical trade-off between exploitation and exploration; the seller needs to simultaneously gain information about the customer's preferences and offer revenue-maximizing assortments. The presence of a finite inventory significantly changes how the seller is supposed to behave in the face of uncertainty. Intuitively, the \text{informativeness} of an assortment must incorporate the \textit{availability} of the offered products. For example, it is not clear how a rare but valuable item competes against an abundant but less valuable item, especially when \textit{utility} is unknown and has to be learned.
%REVIEW%
%This problem is originally motivated by online retail, where the retailer is often not able to simultaneously display all the available products to its customers and has to select only a restricted subset of products to offer. In most realistic settings,  Moreover, the seller is often subject to inventory constraints, which limit the number of available units of each product. The objective is to learn the model parameters in order to maximize the seller's expected revenue over the T periods, which is equivalent to minimizing the seller's regret. 

%We consider a seller with N products over a sales horizon of T periods. In each period $t$, a new customer arrives and the seller needs to decide on a subset $S\subseteq [N]$ of products to offer, with cardinality $|S|\leq K$. The seller then observes the purchase decision of the customer, which is governed by a multinomial logit choice model (MNL) (\cite{Plackett}, \cite{Mcfadden1977ModellingTC}). This is a widely used choice model in the revenue management literature. More precisely, given the offered assortment $S$, the probability that the customer purchases product $i\in S$ is given by:
%\begin{equation*}
    %p_i(S) = \frac{v_i}{1+\sum_{j\in S} v_j}
%\end{equation*}
%where $\{v_i\}_{i\in [N]}$ are unknown model parameters. We would like to underline that the seller only observes the purchase decision of the customer $t$, which depends on all the products present in the set offered at time $t$. 

More formally, we consider the MNL-bandit problem with knapsack constraints, where a seller sells $N$ products over a time horizon of $T$ periods. Each product $i\in [N]$ is associated with a revenue $r_i$ and has an initial inventory level of $q_i T$ units, which cannot be replenished once it has been exhausted. In each period $t$, a customer arrives and then offers an assortment $S_t\subseteq [N]$ selected from a constrained subset $\mathcal{S}$ of assortments.
%REVIEW% %The set of feasible assortments we consider is $\mathcal{S} = \{S\subseteq[N]\mid |S_t|\leq K\}$, for some $K\leq N$. 
The customer purchases $c_t\in S_t\cup \{0\}$ according to a multinomial logit (MNL) choice model with unknown parameters. Here, 0 denotes the no-purchase option. The MNL choice model, first introduced in \cite{Luce59},  \cite{Plackett} and \cite{Mcfadden1977ModellingTC}, can be described as follows: given the assortment $S$, the probability that product $i$ is purchased is given by:
\begin{align*}
    \pi(i, S) = \begin{cases}
    \frac{v_i}{v_0+\sum_{j\in S} v_j} \qquad &\text{ if $i \in S \cup \{0\}$}\\
    0 \qquad &\text{ otherwise,}
    \end{cases}
\end{align*}
where $\{v_i\}_{i\in [N] \cup \{0\}}$ are model parameters unknown to the seller. The sale of product $i$ decreases by one unit the available amount of product $i$. When set $S$ is offered, the purchase probability of product $i\in S$ is equal to $\pi(i, S)$, while the expected revenue $R(S)$ of assortment $S$ can be written as follows:
\begin{equation*}
    R(S) = \sum_{i\in S} r_i \pi(i, S).
\end{equation*}

We would like to mention that even in the case where the model parameters are known to the seller, we are not aware of any computationally efficient algorithm to compute the optimal dynamic policy that respects the hard inventory constraints. Note that this contrasts with the static MNL assortment problem without inventory constraints, where optimal solutions are well characterized (see for instance \cite{Talluri}, \cite{Rusmevichientong2010}, \cite{Dsir2014NearOptimalAF} and \cite{Davis2013AssortmentPU}). Following previous works on bandits with knapsack constraints (for example, \cite{Agrawal_global_bandits} and \cite{Badanidiyuru}), we thus consider a stronger benchmark, where the hard inventory constraints are relaxed. This benchmark is based on an exponential-size LP (described more precisely in Section \ref{sec:preliminaries}).
%REVIEW%
%The static assortment optimization problem, where the underlying MNL model parameters are known to the seller, has been well studied (\npcomment{cite desir...}).

\subsection{Our contributions}
\vspace{0.2cm}

We present MNLwK-UCB, a UCB-based algorithm for the MNL-bandit problem with knapsack constraints. 
% For $r_{\max} = \max_{i \in [N]} r_i$ the highest revenue amongst all products, $r_{\sf opt} \in [0, r_{\max}]$ the optimal revenue in the time independent LP relaxation of the problem, $H_{\boldsymbol{q}} = \sum_{i \in [N]} q_i^{-1}$ the inventory cost, and $r_{\sf norm} \in [0, r_{\max}]$ is the optimal revenue in the LP relaxation normalized with respect to the inventory (See the next section for a more detailed discussion),
Our main result is a characterization of the regret of MNLwK-UCB, as we show that 
\small
\begin{equation}\label{eq:regretbound}
    \text{Regret}_T \lesssim N r_{max}+ \left(r_{\sf inv} + \|\boldsymbol{v}\|_2 r_{\sf opt} + \|\boldsymbol{v}\cdot\boldsymbol{r}\|_2\right)\sqrt{T},
\end{equation}
\normalsize
(where $\lesssim$ hides numerical constants and polylogarithmic factors in $N$ and $T$). Here, $r_{\sf opt}$ is the optimal revenue in the fluid relaxation, and $r_{\sf inv}$ captures the cost of having a finite inventory. Both quantities will be stated in detail in the next section. The bound stated in Inequality \eqref{eq:regretbound} allows us to compute the regret depending on the size of the inventory, which is represented by the growth of $\boldsymbol{q}$ in $T$.

We show that under the large-inventory setting (i.e., all initial capacities $q_i T$ grow quasi-linearly with the time horizon $T$), our regret bound becomes $\Tilde O\left(N + \sqrt{NT}\right)$, which recovers the  known infinite inventory bound. We also show that for the smaller inventory settings, this regret rate is maintained. In particular, in the sub-linear growth setting (i.e., when capacities satisfy $q_i = \Theta(T^{\alpha})$, $\alpha < 0$), the regret bound is in $\tilde{O}\left(N\left(1 + T^{\frac{1 - \alpha}{2}}\right)+ \sqrt{NT}\right)$. 

Note that these bounds do not depend on the complexity of the constraints set $\mathcal{S}$. In particular, when $\mathcal{S}$ represents cardinality constraints of cardinal at most $K$, we achieve a $K$-free regret bound. 

To the best of our knowledge, these bounds are the best known bounds in all the capacity settings we consider. We summarize these results in Table \ref{contributions}, where we also compare with the existing work for the particular case where the constrained set $\mathcal{S}$ consists of cardinality constraints.

\small{
\begin{table}[h]\label{contributions}
\centering
\resizebox{\columnwidth}{!}{
\begin{tabular}{c|c||c|c}
 \textbf{Inventory size}& \textbf{Definition}& \textbf{Regret} & \textbf{Prior best bound} \\
\hline
Quasi-linear growth &$q_{\min}^{-1} = \Tilde{O}(1)$ & $\tilde{O}\left(N + \sqrt{NT}\right)$& $\Tilde{O}(N K^{5/2} + \sqrt{NT})$ (\cite{miao2021general})\\
\hline

\hline
Sub-linear growth &$q_{\min}^{-1} = \Theta(T^\alpha), \alpha \in (0, 1]$ & $\tilde{O}\left(N + N T^{\frac{1-\alpha}{2}} + \sqrt{NT}\right)$& $\Tilde{O}(K^{2/3} N^{1/3}) T^{2/3} )$  for $\alpha = 2/3$ (\cite{Cheung_assort_inventory})\\
\hline

\hline
\end{tabular}}
\label{tableresults}
\caption{Summary of contributions}
\end{table}}
\normalsize

Note that our work is closely related to the Bandits with Knapsacks problem (see for instance \cite{Agrawal_global_bandits}, \cite{Badanidiyuru}, \cite{pmlr-v45-Xia15} and \cite{DBLP:journals/corr/XiaLQYL15}), for which algorithms achieving optimal regrets in the linear-growing inventory regime have been recently proposed. However, these results are not directly applicable to our setting, as it would involve considering each possible assortment as an individual arm in the multi-arm bandits framework. This would yield an exponential action space, and would therefore not lead to a tractable policy. On the contrary, by using the specific structure of MNLwK-UCB, combined with the specific properties of the MNL model, we are able to formulate an exponential LP similar to the one used in \cite{Agrawal_global_bandits}, for which we can design an efficient algorithm.

The idea behind MNLwK-UCB consists of mimicking the behavior of the optimal fluid policy that knows the true utilities $\boldsymbol{v}$. While this idea is simple, the main challenge is to satisfy simultaneously \textit{computational tractability} and \textit{low regret}. Intuitively, maintaining computational tractability requires dealing with the exponential size of the action space (that is growing even more with inventory constraints), which might require compressing the information on which the decision is based. More often than not, this compression leads to a sub-optimal learning of the optimal fluid policy, which itself leads to sub-optimal regret. On the other hand, maintaining low regret requires an efficient learning of the optimal fluid policy, which in turn cannot be done without maintaining enough information about this policy, which might mean confronting the large size of the action space. With this in mind, overcoming the main challenge requires compressing the information on the optimal fluid policy the right way.

We overcome this trade-off by solving the fluid policy of an alternative \textit{accessible} choice function that uniformly approaches $\pi$. The alternative choice function $\pi^{\sf UCB}$ represents a high probability upper confidence bound on the true $\pi$, and is constructed in an online fashion based on collected data. Playing the optimal fluid policy of an alternative choice function has the benefits of isolating the computational challenges of the problem from its learning challenges, and defers the former to the abundant offline assortment optimization literature, allowing us to focus on the latter. 

Our main contribution lies in the approach we take in our regret analysis. We view the interaction between the seller and the customer as a repeated game, where the new goal of the seller is not to maximize revenue, but to play the optimal fluid policy as fast as possible, and the moves of the customer are constrained by the unknown MNL parameters. This new goal implicitly protects the seller against inventory mis-management, as the target fluid policy correctly handles the capacity constraints. From this modified point of view, measuring the regret of a policy is conducted by measuring the consumption rate of each product, and a good policy should match the (unknown) optimal consumption rates as fast as possible. By these standards, we show that the performance of any policy can be characterized by three (possible) sources of revenue loss. The first one being the cost of playing a randomized strategy on the seller's side (i.e., the seller plays not necessarily an assortment but a sample from a distribution of assortments), the second one being the cost of randomness in the customer's choice (i.e., the seller does not observe $\pi$, but only a few realizations of $\pi$), and the third one being the \textit{cumulated} mis-specification of the optimal consumption rates through the game (i.e., the seller plays with a -constantly improving- belief on what the optimal consumption rates are, instead of what they truly are). 

Decomposing the performance into these three independent losses provides a \textit{modular} approach to tackle the \textit{utility-revenue-capacity} trade-off in uncertain environments. The first loss depends entirely on the steadiness of the seller's decisions, and it suggests that the revenue decreases with the variability of the seller's decisions across the time horizon. In other words, the more consistent a seller is in their choices, the better the policy (all else being equal). Even though the second loss seems to depend on the choice model only (in particular, it seems algorithm-agnostic), it suggests that offering the same assortments repeatedly has a regularization effect that is beneficial to the seller. The third loss captures the learning aspect of the problem. It puts an equal weight on time steps. In particular, it discards strategies that spend an unbalanced inventory at first, or strategies that conserve a fraction of their inventory only to attempt to deplete them at the late stages of the sale. It also suggests that what matters is not so the seller's knowledge of the choice function (exponential in size), as much as their knowledge of the optimal consumption rates (linear in size). In our information-compression analogy stated earlier this section, this third loss suggests that the optimal consumption rates are "the right compressor" for the inventory problem.

MNLwK-UCB is constructed so that all three losses are as low as possible. While the first and second losses are insured by the convergence of the distribution that is played, it is to no surprise the third loss that is hard to minimize, as it encapsulates the heart of the problem's trade-off. In particular, one parameter that is hard to bypass is the complexity of the constrained assortment subset $\mathcal{S}$. By leveraging structural properties that are specific to the MNL model, we successfully yield a regret bound that does not depend on the choice of $\mathcal{S}$ for linear-growing inventories, matching existing bounds for the infinite inventory setting.

%REVIEW% %On the other hand, the dynamic assortment optimization problem was first considered by [Caro and Gallien], under the assumption of independent demand across the products present in each assortment. More recent works incorporated a MNL choice model into the dynamic assortment problem (see for instance ...). 
% in the MNL-Bandit problem.
%REVIEW% %In particular, the inventory balancing policy of \cite{Golrezaei} achieves a tight $1-1/e $ approximation guarantee for the adversarial arrival model and a 3/4-approximation guarantee for the i.i.d. arrival model in the large-inventory regime. 
%In these works, customers may belong to heterogeneous segments, whereas in our setting, the customers make purchase decisions according to the same MNL model.
%which are assumed to follow the same unknown MNL model. Instead, we 
%Our work is a special case of this model, with customers making purchase decisions according to the same unknown MNL model.
%We, however, leverage the special structure of the MNL model to design a tighter algorithm than the one that would result from \cite{Golrezaei}. 

\subsection{Related work}
\vspace{0.2cm}
\noindent\textbf{Dynamic assortment optimization.}
The dynamic assortment optimization problem was first considered by \cite{Caro_gallien}, under the assumption of independent demand across the products present in each assortment. More recent works incorporate a MNL model into the problem (see for instance \cite{Rusmevichientong2010}, \cite{Saure_zeevi2013}, \cite{Wang2018}, \cite{Chen_unconstrained} \cite{AgrawalAGZ17_TS} and  \cite{DBLP:journals/corr/AgrawalAGZ17a}). In particular, \cite{Chen_unconstrained} achieves the optimal $\Tilde{O}(\sqrt{T})$ regret bound in the unconstrained setting. When there is an extra cardinality constraint on the valid assortments, \cite{DBLP:journals/corr/AgrawalAGZ17a} and \cite{AgrawalAGZ17_TS} propose respectively UCB and Thompson-sampling based policies which achieve near-optimal $\widetilde O(\sqrt{NT})$ regrets matching the information theoretic lower bounds.
%REVIEW% %thus matching the $\Omega(\sqrt{NT})$ lower bound exhibited in \cite{Wang2018}.
However, this prior stream of work does not address inventory constraints. Moreover, we would like to note that MNLwK-UCB estimates the unknown utilities $\boldsymbol{v}$ based on the sampling-based estimation presented in \cite{DBLP:journals/corr/AgrawalAGZ17a} for the MNL-bandit problem.

Our problem is formulated and studied in \cite{Cheung_assort_inventory}, where the authors introduce the MNL-bandit problem under inventory constraints and cardinality constraints $K$ (meaning the size of the selected assortment should not exceed $K$).
% consider the MNL-bandits problem with inventory constraints. 
Additionally, the authors assume that the seller has knowledge of a radius $R>0$ satisfying $R^{-1} \leq v_{\min} \leq v_{\max} \leq R$, as well as a large inventory assumption $q_{\min}^{-1} \lesssim T^{2/3}$ and propose two algorithms: a computationally efficient online policy, which incurs a regret $\Tilde O\left((RK)^{2/3}T^{2/3}\sqrt{N}\right)$, and a UCB-based algorithm, which achieve a near optimal $\Tilde O\left(R^{3} K^{5/2}N\sqrt{T}\right)$ regret in the large-inventory regime. However, this last policy is not computationally efficient as it requires solving an exponential-size LP at each step. In our work, we study a more general constrained problem (where we consider subsets within a set $\mathcal{S}$). We do not make any lower bound assumptions on $\boldsymbol{v}$, and our approach works without any assumptions on the initial inventory). Our main methodological novelty consists in considering a different exponential-size LP, based on optimistic confidence bounds instead of exploration bonuses, and show that it can be efficiently solved, while still achieving a better regret guarantee as in \cite{Cheung_assort_inventory}. 
 
Finally, we would like to mention \cite{miao2021general}, a work conducted at the same time. One of the models they tackle is exactly the cardinality-constrained assortment selection question under the MNL model, which corresponds to the special case where $\mathcal{S}$ represents assortments of bounded cardinality. Both of our approaches involve UCB estimators as a proxy, and an epoch-based offering method. While \cite{miao2021general} adopt a primal dual optimization approach, we resolve a LP fluid relaxation before each epoch. It is worthy to note that both approaches are efficient (polynomial in $N$ and linear/quasi-linear in $T$). Moreover, \cite{miao2021general} adopt a fixed assortment update budget (i.e., updating assortments once after every $L$ periods) where $L$ depends on the upper bound on $\boldsymbol{v}$, as well as the cardinality constraint bound $K$. Consequently, their approach achieves a regret bound of $\Tilde{O}(\sqrt{NT} + N K^{5/2})$ in the regime where the inventory grows quasi-linearly, compared to our $\Tilde{O}(\sqrt{NT} + N)$ bound. Our work uses a variable-length epoch approach (until a no-purchase occurs) that does not depend on any prior knowledge on $\boldsymbol{v}$, or the constrained set $\mathcal{S}$. Consequently, MNLwK-UCB can be easily adapted to the case where no assumptions on $\boldsymbol{v}$ are made. Moreover, MNLwK-UCB applies to a large class of constrained sets, and the resulting regret bound we obtain does not depend on the choice of $\mathcal{S}$. Finally, our work applies without any assumption on the size of the inventory.

\vspace{0.2cm}
\noindent\textbf{Online resource allocation and Online Assortments.} Our problem is also closely related to the setting of online assortment and online resource allocation. Specifically, the online assortment with inventory constraints problem has been extensively studied when the seller can observe in each period the arriving customer's 'type' (which specifies the choice model governing the purchase decisions). When the customers belong to different segments with known arrival rates, this problem is the well-known choice-based network revenue management problem, for which approximation algorithms have been proposed (see for instance \cite{gallego2015online}, \cite{Kunnumkal}, \cite{MEISSNER2012459} and \cite{Ma_2019}). A few recent works focus on a setting with uncertainty on the future customers' types (see for instance \cite{Golrezaei}, \cite{Ma2018}, \cite{Chen2016AssortmentPF} and \cite{Bernstein2015}). In particular, the inventory balancing policy of \cite{Golrezaei} achieves a tight $1-1/e $ approximation guarantee for the adversarial arrival model and a 3/4-approximation guarantee for the i.i.d. arrival model in the large-inventory regime, where $T/\min q_i = k$ for some positive integer $k$. However, in our setting, we do not observe the customer's 'type', but need to learn it over the $T$ periods. Therefore, the problem we consider combines aspects of the MNL-Bandit with online assortments. 

We would like to comment on our regret bound in the light of known results in the online resource allocation framework, which considers the problem of matching an online sequence of requests to some budgeted agents. A key parameter in these works is the ratio $\gamma$ between the value of a single request and the agents' budgets. This ratio is usually called the {\em bid to budget ratio} in reference to the Adwords problem (\cite{Mehta2007}). When $\gamma$ is assumed to be small, a $(1-\epsilon)$ competitive ratio can be achieved in the stochastic arrival setting (\cite{Devanur2009}, \cite{Devanur2019}, \cite{Agrawal2014}). In particular, \cite{Devanur2019} recently obtained a $1-\epsilon$ approximation algorithm when $\gamma = O(\log(n/\epsilon)/\epsilon^2)$ and showed that this upper bound on $\gamma$ is almost optimal. In our setting, by interpreting $\frac{\|\boldsymbol{v}\|_2 r_{\sf opt} + \|\boldsymbol{v}\cdot\boldsymbol{r}\|_2}{\sqrt{N}}$ as metric for the bids and $r_{\sf inv}$ as a metric for the budget, we can see the large-inventory assumption for which we achieve a $\widetilde O(\sqrt{NT})$ regret as an analogue of the small-bid assumption in the resource allocation framework. It is not known whether the dependency of our regret bound in $\boldsymbol{v}$, $\boldsymbol{q}$ and $\boldsymbol{r}$ is tight.

\vspace{0.2cm}
\noindent\textbf{Budgeted bandits.} Another closely related stream of literature is the budgeted bandits problem. This is an extension of the MAB problem where pulling an arm also generates a cost and the agent is subject to a budget constraint (see for instance \cite{Tran2012}, \cite{Antos2008} and \cite{bubeck2010pure}). In the case of random costs, this problem is referred to as the Bandits with Knapsacks problem (first introduced in \cite{Badanidiyuru}). This problem originated from the study of dynamic pricing with limited supply (see for instance \cite{BesbesZ09}), which can be cast as as special case of BwK. \cite{Agrawal_global_bandits} recently showed that a natural extension of the UCB family of algorithms achieves a near-optimal regret for a variant of BwK with more general resource constraints. Our work can be interpreted in this framework by considering each assortment as an independent arm.  However, as mentioned above, a direct application of the techniques from \cite{Badanidiyuru} and \cite{Agrawal_global_bandits} in our setting would lead to a computational complexity and a regret bound both linear in the number of assortments, which is exponentially large.

\vspace{0.2cm}
\noindent\textbf{Combinatorial bandits and combinatorial semi-bandits with knapsacks.} Our problem is also closely related to the combinatorial bandits literature (see for instance \cite{Kveton}, \cite{Chen_combinatorial} and \cite{Qin_combinatorial}), where in each period, the agent needs to decide on a subset of arms to pull (called a superarm), and obtains a rewards which is a function of the individual arms' rewards. Closest to our work is the combinatorial bandits with knapsacks problem, which is a generalization of combinatorial bandits that incorporates resource constraints in a similar way to the BwK problem. This problem has been studied in \cite{CSBK} and mentioned as an application of the more general results of \cite{ImmorlicaSSS19} and \cite{AdvancesInBwK}. Our problem can be translated to this framework, by interpreting each assortment as a superarm. However, in our setting, the individual arms' rewards depend on all the arms present in the subset offered, whereas the rewards are assumed independent in the combinatorial bandits framework. Furthermore, the revenue generated by an assortment is not even monotonic, as introducing new products in the assortment may reduce the probability of purchasing the most profitable items and may thus impair the expected revenue. Hence, we cannot directly use the existing combinatorial bandits algorithms to obtain a low regret policy.
\section{Dynamic assortment policy}

\subsection{Preliminaries}
\label{sec:preliminaries}
\vspace{0.2cm}

\noindent\textbf{The assortment model.} An instance of the contrained Multinomial logit Bandits with Knapsacks (MNL-BwK) is defined by the parameters  $(T, N, \mathcal{S}, \boldsymbol{v}, \boldsymbol{q}, \boldsymbol{r})$ where $T$ is the length of the time horizon, $N$ is the number of products, $\mathcal{S} \subset [N]$ is the constrained set of possible assortments. $\boldsymbol{v}$ is the choice model's parameter vector, $T \boldsymbol{q} \geq 0$ is the capacity vector, $\boldsymbol{r}$ is the price vector. Dividing all the parameters with $v_0$ does not change the probabilities, and we can assume without loss of generality that $v_0 = 1$, this way without ambiguity we denote $\boldsymbol{v} = (v_1, \ldots, v_N) \in \mathbb{R}_+^{N}$. 

For simplicity of exposition, we restrict our attention to the case where the most likely outcome is the no-purchase option. This is equivalent to the following assumption on $\boldsymbol{v}$:
\vspace{-2mm}
\begin{assumption}\label{assumption:vmax}
$\|\{v_i\}_{i \in [N]}\|_{\infty} = v_{\max}\leq v_0 = 1$.
\end{assumption}

\vspace{-2mm}
The above assumption is fairly general and implies that the probability of no-purchase is at least as high as the probability of purchasing any other product. This is particularly true in the setting of online retail where conversion rates (ratios of sales to impressions) are small. This assumption can be relaxed using \cite{DBLP:journals/corr/AgrawalAGZ17a}.  We would like to note that unlike \cite{Cheung_assort_inventory}, we do not require any knowledge of a lower bound on the utilities $\{v_i\}_{i \in [N]}$. 

Since each product can be purchased at most $T$ times, without any loss of generality we can bound the capacity $T \boldsymbol{q}$ by $T$, or equivalently $\boldsymbol{q} \in [0, 1]^N$. This way, $q_i = 0$ means that product $i$ is not available, while $q_i = 1$ means that product $i$ has non-binding capacity. 
At each time $t \in [T]$, the seller offers assortment $S_t \in \mathcal{S} \subset [N]$, and the customer makes a decision $c_t \in S_t \cup \{0\}$ distributed according to $\pi(i, S_t)$. In the case where $c_t \neq 0$, the customer pays $r_{c_t}$. In the case where $c_t = 0$, the customer does not receive any payment. $\mathcal{S}$ represents the constrained set of possible assortments. 

The only formal assumption we make on $\mathcal{S}$ is that the time-independent offline LP relaxation (introduced in the following paragraph) can be solved efficiently. This assumption is minimal to our problem since it is computationally as hard as the LP relaxation version. This assumption encapsulates several naturally arising constraints over the assortments that the seller can offer. These include cardinality constraints (where there is an upper bound on the number of products that can be offered in the assortment), partition matroid constraints (where the products are partitioned into segments and the seller can select at most a specified number of products from each segment) and joint display and assortment constraints (where the seller needs to decide both the assortment as well as the display segment of each product in the assortment and there is an upper bound on the number of products in each display segment). In general, it encapsulates all $\mathcal{S}$ satisfying
\begin{equation*}
    \mathcal{S} = \{S \subset \{1, \ldots, N\} | \boldsymbol{A} \boldsymbol{x}(S) \leq \boldsymbol{b}, \boldsymbol{0} \leq \boldsymbol{x} \leq \boldsymbol{1}\} 
\end{equation*}
where $x_i(S) = 1$ is and only if $i \in S$. $\boldsymbol{A}$ is a Totally Unimodular matrix, and $\boldsymbol{b}$ is an integral vector. This way of constraining assortments offers a rich class of planning problems including the examples discussed above. \cite{davis2013assortment} provides a detailed discussion on these problems. 
\vspace{0.2cm}

\noindent\textbf{Benchmark.} We are interested in non-anticipative policies where the set offered at time $t$ only depends on the history up to time $t$: $\{S_{s}, c_{s}\}_{s = 1, \ldots, t-1})$. We compare ourselves to a clairvoyant policy which has access to the true parameter $\boldsymbol{v}$ (but not the choice realizations of the customers). Consider the following time-independent linear program:
\begin{align*}\label{opt:exponential}
    {\text{LP$(\pi)$}}\qquad r_{\sf opt}(\pi, \boldsymbol{r}, \boldsymbol{q})= \text{max }& \sumS y(S)\cdot R(S|\pi) \\
    \text{s.t }&\sumS y(S) \cdot \pi(i,S) \leq q_i, \; \; \; \forall i = 1, \ldots, N \\ &\boldsymbol{y} \geq 0 \; \;, \; \sumS y_S \leq 1.
\end{align*}
Following the discussion in the previous paragraph on the feasible assortments $\mathcal{S}$, we formally make the following assumption on $\mathcal{S}$:
\begin{assumption}\label{assumption:constraints}
For any given $\tilde{\pi}$, the program \text{LP$(\tilde{\pi})$} can be solved in polynomial time in $N$.
\end{assumption}
LP$(\pi)$ is a linear program with exponentially many variables, and a linear number (in $N$) of constraints. Consequently, whether Assumption \ref{assumption:constraints} holds depends on $\mathcal{S}$ only. $\text{LP$(\pi)$}$ is also time-independent. Lemma \ref{lem:OPT-LP} provides an upper bound on the expected revenue of any feasible dynamic assortment policy, formalized below.
\begin{lemma}
\label{lem:OPT-LP}
For any non-anticipative policy $\mathcal{P}$, $\mathbb{E}_{\mathcal{P}}\sum_{t = 1}^T r_{c_t} \leq T r_{\sf opt}$.
\end{lemma}

We give the proof of Lemma \ref{lem:OPT-LP} in  Appendix \ref{app:lemmas}. Let $y^{\sf OPT}$ be an optimal solution of $\text{LP$(\boldsymbol{v})$}$. Motivated by Lemma \ref{lem:OPT-LP}, we consider as a benchmark the algorithm which plays in each period $t$ a set $S_t$ drawn according to $y^{\sf OPT}$. For such an algorithm, $q^{\sf OPT}_i = \sum_{S \in \mathcal{S}} y^{\sf OPT}(S) \pi(i, S) \leq q_i$ represents the \textit{optimal consumption rate} of product $i$, and the corresponding unitary revenue is $r_{\sf opt}$, satisfying $r_{\sf opt} = \sum_{i \in [N]}r_i q_i^{\sf OPT}$. Note that the solution returned by this algorithm is feasible only in expectation, whereas the feasible policies in our setting are subject to hard inventory constraints. We would like to point out that this type of linear program has been often considered in the revenue management literature and is commonly referred to as the (choice-based) deterministic linear program \cite{Gallego04managingflexible}. This strong benchmark was for instance used by \cite{Golrezaei} in the setting of dynamic assortment optimization, and by \cite{Agrawal_global_bandits} and \cite{Badanidiyuru} in the more general context of bandits with knapsacks. 

Finally, we introduce \textit{the inventory's cost} $r_{\sf inv}$:
\begin{equation}\label{eq:norm}
    r_{\sf inv} = \sum_{i \in [N]}r_i \frac{\sum_{S \in \mathcal{S}}y^{\sf opt}(S)\pi(i, S) }{\sqrt{q_i}}= \sum_{i \in [N]}r_i \frac{q^{\sf OPT}_i}{\sqrt{q_i}}.
\end{equation}
A more detailed discussion on the intuition behind $r_{\sf inv}$ will be developed in the next section.

\medskip
\noindent\textbf{Objective.} The objective of the seller is to minimize its cumulative expected regret, which is expressed as:
\begin{equation*}
    \text{Regret}(T) = T r_{\sf opt} - \mathbb{E}
    \left[\sumt R(S_t)\right].
\end{equation*}

\subsection{Our algorithm: MNLwK-UCB}
\label{sec:policy_description}
\vspace{0.2cm}

We first give an outline of our dynamic assortment policy, which we name MNLwk-UCB. We divide the time horizon into epochs, where each epoch $\ell$ is composed of $\mathcal{E}_{\ell}$ consecutive time steps. At the beginning of each epoch $\ell$, we decide on the assortment $S_{\ell}$ to offer to arriving customers repeatedly until we observe a no-purchase. This stopping rule is based on~\cite{DBLP:journals/corr/AgrawalAGZ17a} and allows to obtain unbiased estimators of the unknown MNL parameters. In particular, the average number of times each product has been purchased in an epoch is an unbiased estimator of the MNL parameter corresponding to the product. Based on these estimators, we construct at the beginning of each epoch $\ell$ confidence bounds on the utilities, from which we derive upper confidence bounds on the choice probability $\pi$ for each product $i$ and assortment $S \in \mathcal{S}$. We then compute a distribution $\{y_{S,\ell}\}_{S\subseteq [N]}$ and sample the assortment $S_{\ell}$ for epoch $\ell$ based on the distribution. Note that the sampling of $S_{\ell}$ is done only once, at the beginning of each epoch $\ell$. Naively computing the distribution $\{y_{S,\ell}\}_{S\subseteq [N]}$ and sampling from it may have exponential complexity. However, due to the structure of the LP we solve, this can be done efficiently. 

We denote by $q^t_i \in [0,1]$ the remaining inventory fraction for each product $i$ at time $t$. After observing the purchase decision $c_t$ of the customer, we decrease the inventory of $c_t$ (corresponding to $T q_i^t$) by one unit. The algorithm stops as soon as any of the products $i$ is out of stock (i.e. $q^t_{i} = 0$). This condition ensures that the algorithm is always feasible. Even though it seems somewhat restrictive to stop as soon as one product depletes (as the seller might be tempted to use the remaining time steps to sell other available products) we show that the associated cost is negligible. In fact, with high probability, no product is depleted before the end of the time horizon.

Note that the procedure described above is not yet a valid algorithm as it involves possibly two computational bottlenecks. The first one being solving an LP with exponentially many variables. The second one being computing and sampling from a distribution with possibly exponentially sized support. Both of these issues are simultaneously tackled. We present a computationally efficient algorithm which returns an optimal distribution without having to  explicitly use the lower and upper bounds for each set $S$. Furthermore, the distribution returned by MNLwK-UCB has a linear size support and therefore, we can sample efficiently from it. 

We are now ready to describe the details of the algorithm.

\noindent\textbf{Confidence bounds for $\boldsymbol{v}$.} We now present the details on the construction of the utility estimator, and its high probability confidence bounds. We would like to mention that we do not follow the classical approach to design  maximum likelihood estimators (which has been considered for instance by \cite{Cheung_assort_inventory} in their dynamic assortment policy under inventory constraints). The main limitation of the MLE approach is that it requires the knowledge of a lower bound on the values of the true utilities, which may not be available in practical applications. In order to avoid this restrictive assumption, we choose to design our confidence intervals based on the sampling method presented in \cite{DBLP:journals/corr/AgrawalAGZ17a}. We describe it here for the convenience of the reader. 

We define $n_{i,\ell}$ the number of customers who purchased $i$ in epoch $\ell$, for all $i \in S_\ell$:
\begin{equation*}
    n_{i,\ell} = \sum_{t \in \epochl} \mathbbm{1}(c_t = i),
\end{equation*}
and we define $N_i(\ell)$, the number of times product $i$ has been purchased up to the end of epoch $\ell$:
\begin{equation*}
    N_{i}(\ell) = \sum_{\tau \leq \ell}n_{i, \tau} = \sum_{\tau \in \mathcal{T}_i(\ell)} n_{i, \tau},
\end{equation*}
where $\mathcal{T}_i(\ell) = \left\{\tau \leq \ell \, \middle| \, i \in S_\ell\right\}\text { and } T_i(\ell)=|\mathcal{T}_i(\ell)|$ denote respectively the set of epochs and the number of epochs which have offered product $i$ so far. We then compute $\tilde{v}_{i,\ell}$ the average number of times each product $i$ is purchased per epoch:
\begin{equation*}
    \tilde{v}_{i,\ell}= \frac{N_i(\ell)}{T_i(\ell)} = \frac{1}{T_i(\ell)} \cdot \sum_{\tau \in \mathcal{T}_i(\ell)} n_{i,\tau},
\end{equation*}
\noindent At the end of epoch $\ell$, we set:
\begin{equation*}
    \vucb_{i,\ell} = \min\left\{\tilde{v}_{i,\ell} + \Bigg( \displaystyle \sqrt{\tilde{v}_{i,\ell}\frac{48\log{({\sqrt{N}}\ell+1)}}{T_i(\ell)}} + \frac{48\log{({\sqrt{N}}\ell+1)}}{T_i(\ell)} \Bigg), 1\right\},
\end{equation*}
\begin{equation*}
\vlcb_{i,\ell} = \max\left\{\tilde{v}_{i,\ell} - \Bigg( \displaystyle \sqrt{\tilde{v}_{i,\ell}\frac{48\log{({\sqrt{N}}\ell+1)}}{T_i(\ell)}} + \frac{48\log{({\sqrt{N}}\ell+1)}}{T_i(\ell)} \Bigg), 0\right\}.
\end{equation*}
\noindent We will show that with high probability, $v_{i} \in [\vlcb_{i, \ell}, \vucb_{i, \ell}]$.

\noindent\textbf{Optimistic assortments.}  The previous confidence bounds induces the following natural upper confidence bound for the choice function $\pi$:
\begin{equation}\label{piucb}
    \pi^{\sf UCB}_\ell(i, S) := \frac{\vucb_{i, \ell}}{1 + \sum_{j \in S} \vlcb_{j, \ell}}\mathbbm{1}(i \in S).
\end{equation}
In epoch $\ell$, MNLwK-UCB computes the distribution over sets which maximizes revenue for the estimates of the purchase probabilities. More precisely, the distribution $\{y_{\ell}(S)\}_{S \in \mathcal{S}}$ played in epoch $\ell$ has a linear support that is known to the seller and is a solution of the following program:
\begin{align*}\label{opt:exponential}
\text{LP$(\pi^{\sf UCB}_{\ell-1})$}:\qquad \text{max }& \sumS y_S \sum_{i \in [N]} r_i  \cdot \pi^{\sf UCB}_{\ell - 1}(i, S) \\
    \text{s.t }&\sumS y_S \cdot \pi^{\sf UCB}_{\ell - 1}(i, S) \leq (1 - \omega_{i, \ell - 1}) q_i \; \; \; \forall i = 1, \ldots, N \\ &y \geq 0 \; \;, \; \sumS y_S \leq 1,
\end{align*}
\noindent where $\boldsymbol{\omega}>0$ is a shrinkage factor. Adding the shrinkage factor $\boldsymbol{\omega}$ ensures that stock-out does not happen with high probability, while accounting for the different errors possibly induced in this approximation in order to respect the hard inventory constraints (from discrete to fluid over both assortment model and the randomization over offered assortments, and the error resulting from mis-specifying $\boldsymbol{v}$). $\boldsymbol{\omega}$'s expression is derived naturally from the regret analysis:
\begin{equation}\label{eq:omega}
    \forall i \in [N], \quad \ell \geq 0, \qquad \omega_{i} =  \frac{A_0}{q_i T} + \frac{A_1}{\sqrt{q_i T}},
\end{equation}
where $A_0$, and $A_1 = \Tilde{O}(1)$ are numerical constants times polylogarithmic factors in $T$. The exact expressions of $A_0$ and $A_1$ are explicited in the regret analysis section. For clarity of exposition, we assume LP$(\pi^{\sf UCB}_{\ell - 1})$ to be feasible, which is equivalent to making the following additional assumption on the minimal size of the inventory:

\begin{assumption}\label{assumption:qmin}
    $\max_{i} \omega_{i} < 1$.
\end{assumption}
Assumption \ref{assumption:qmin} is essentially equivalent to $q_{\min}^{-1} \lesssim T$, which is the same as saying that the inventory has at least a polylogarithmic growth in $T$. It easily relaxed by noticing the following: Assume there is a product $i$ which grows sub-logarithmically. Disregarding such a product will have minimal impact on the final regret. Stronger assumptions on $q_{\min}$ are made in both \cite{miao2021general} ($q_{\min}^{-1} \lesssim 1$) and \cite{Cheung_assort_inventory} ($q_{\min}^{-1} \lesssim T^{1/3}$).

The set $S_{\ell}$ played throughout epoch $\ell$ is then drawn according to the distribution $\{y_{\ell}(S)\}$. It is worth to recall that the procedure by which we derive $y_\ell$ gives a distribution with linear support in $N$ that is known to the seller. As a consequence, all samplings are done tractably and efficiently.

Note that by interpreting each assortment $S$ as an independent arm in the multi-arm bandits framework, the program $\text{LP$(\pi^{\sf UCB}_\ell)$}$ plays the same role as the LP used by the UCB-based algorithm presented in \cite{Agrawal_global_bandits} for the Bandits with Knapsacks (BwK) problem. The major difference is that in our setting, problem $\text{LP$(\pi^{\sf UCB}_\ell)$}$ involves an exponential number of variables and therefore it cannot be directly solved. However, we still show that $\text{LP$(\pi^{\sf UCB}_\ell)$}$ can be solved efficiently (linearly in $T$, and polynomial in $N$). The details of MNLwK-UCB are described in Algorithm \ref{main_algo}.

We are now ready to present our regret bound.
\vspace{-0.1cm}
\begin{theorem}\label{main_theorem}
Under assumptions $\ref{assumption:vmax}, \text{ }\ref{assumption:constraints},\text{and }\ref{assumption:qmin}$, Algorithm \ref{main_algo} achieves a regret bound of
\begin{equation*}
\text{Regret}_T \leq B_0 N r_{max}+ B_1 \left(r_{\sf inv} + \|\boldsymbol{v}\|_2 r_{\sf opt} + \|\boldsymbol{v}\cdot\boldsymbol{r}\|_2\right)\sqrt{T},
\end{equation*}
where $B_0, B_1 = \Tilde{O}(1)$ are numerical constants multiplied by polylog factors in $N, T$.
\end{theorem}

\noindent Assumption \ref{assumption:vmax} makes the no-purchase option as the most likely outcome given any assortment. It can be relaxed by adapting ideas from \cite{DBLP:journals/corr/AgrawalAGZ17a} by adding an initial exploratory phase that ensures that all products are selected enough times, and one should expect a more general bound that scales with $\max\left(\sqrt{\boldsymbol{v}}, \boldsymbol{v}\right)$. Its purpose is solely to maintain the focus on how the inventory is managed and alleviate any non-central technicalities. Assumption \ref{assumption:constraints} is the tractability of the corresponding LP relaxation and seems to be minimal and necessary, since for a given constrained set $\mathcal{S}$, one should expect the corresponding online MNL-BwK problem to be (computationally) at least as hard as its time independent LP relaxation. Moreover, Assumption \ref{assumption:constraints} is satisfied for most usual cases (cardinality constraints, matroid constraints, ...). Assumption \ref{assumption:qmin} is about the minimal size of the total inventory, and is also made for simplicity of exposition, as it can be relaxed easily by disregarding the products with sub-logarithmic growth.

We now parse the bound in Theorem \ref{main_theorem}. First, the regret bound scales linearly in the revenue vector $\boldsymbol{r}$. Second, the growth of the $\sqrt{T}$ coefficient is $\left(r_{\sf inv} + \|\boldsymbol{v}\|_2 r_{\sf opt} + \|\boldsymbol{v}\cdot\boldsymbol{r}\|_2\right)$. This term balances between the availability of the products and their utilities, and can therefore be seen as analog to the \textit{bid-to-budget} trade-off mentioned in \cite{Mehta2007}, where $r_{\sf inv}$ measures the budget, and $\|\boldsymbol{v}\|_2 r_{\sf opt} + \|\boldsymbol{v}\cdot\boldsymbol{r}\|_2$ measures the bids. Finally, we show that this bound stays in $\Tilde{O}(\sqrt{NT})$ regardless of the size of the inventory, and the complexity of the constraints set $\mathcal{S}$. In particular, when $\mathcal{S}$ represents cardinality constraints of size $K$, our regret bound does not depend on $K$ in the large inventory setting. A more detailed discussion on the role of the inventory in the regret bound is detailed in the next section.
\begin{algorithm}

\begin{algorithmic}
\STATE \textbf{Initialization:} $\epsilon = \frac{1}{T}, \vucb_{i,0} = \vmax, \vlcb_{i, 0} = 0$, 
for all $i=1,\ldots,N$, $c_0 = 0$
\STATE $\ell = 1$ ; keeps track of the total number of epochs.
\STATE $q^0 = q$ ; \text{keeps track of the remaining inventory.}

\For{$t = 1, \ldots, T$}{
    \If{$c_{t-1} = 0$}{
        Sample $S_\ell$ according to the distribution $y_\ell$, where $y_\ell$ is an optimal solution to $\text{LP$(\pi^{\sf UCB}_{\ell - 1})$}$\;
    }
	Offer assortment $S_{\ell}$ and observe the purchase decision ${c_t}$ of the customer \;
	\If{$c_t = 0$ \text{ (i.e., the decision was a no purchase)}}{

		Compute $n_{i,\ell} = \sum_{t \in \epochl} \mathbbm{1}(c_t = i)$, number of consumers who preferred $i$ in epoch $\ell$, for all $i \in S_\ell$ \;
		Update $\mathcal{T}_i(\ell) = \left\{\tau \leq \ell \, \middle| \, i \in S_\ell\right\}, T_i(\ell)=|\mathcal{T}_i(\ell)|$, set and number of epochs until $\ell$ that offered  product $i$ \;
        Update $N_i(\ell) = N_i(\ell) + n_i(\ell)$, the number of times product $i$ has been purchased until epoch $\ell$. \;
		Update $\tilde{v}_{i,\ell}$ = $\displaystyle \frac{N_i(\ell)}{T_i(\ell)}$, sample mean of the estimates \;

		$\ell = \ell + 1$	\;
	}
	\Else{
		Add time $t$ to epoch $\epochl$ and set $q^{t+1}_{c_t} = q^{t}_{c_t} - 1$ and $q^{t+1}_{i} = q^{t}_{i}$ for all $i\in [N]\setminus c_t$ \;
		
		\If{$q^{t+1}_{c_t} = 0$}{
		Stop the algorithm.
		}
    }
}
\end{algorithmic}
\caption{MNLwK-UCB.}
\label{main_algo}
\end{algorithm}

\subsection{On the cost of finite inventory:}
We delve more on the role of the finite inventory in the regret bound. First, Assumption \ref{assumption:qmin} establishes the upper bound $q_{\min}^{-1} \lesssim T$. We would like to emphasize that,  as opposed to similar works, this assumption is not a "large inventory" assumption, and is only made for clarity of exposition.

We will study the impact of this assumption for different rates $\Theta(T^{\alpha})$, $\alpha \in [0, 1]$. Following Theorem \ref{main_theorem}, this impact will be controlled by the term $r_{\sf inv}$. We recall its expression stated in \eqref{eq:norm}:
\begin{equation*}
    r_{\sf inv} = \sum_{i \in [N]} r_i \frac{q_i^{\sf OPT}}{\sqrt{q_i}}.
\end{equation*}

$r_{\sf inv}$ represents the \text{effective} cost of having a finite inventory. To parse $r_{\sf inv}$, recall that for each product $i$, $q_i$ represents the available capacity, and $q_i^{\sf OPT}$ represents the optimal consumption rate. $\frac{q_i^{\sf OPT}}{q_i}$ represents the optimal saturation rate of product $i$. For the particular case of no inventory constraints, $\boldsymbol{q} = 1$, and we have
\begin{equation*}
    r_{\sf inv} = \sum_{i \in [N]} r_i q_i^{\sf OPT} = r_{\sf opt} \leq r_{\max},
\end{equation*}
plugging-in this inequality in Theorem \ref{main_theorem} recovers a regret bound of $\Tilde{O}(N + \sqrt{NT})$. This discussion can be extended to the case where the inventory grows quasi-linearly in $T$, and we would still have:
\begin{equation}\label{eq:rinv}
    r_{\sf inv} \leq\frac{r_{\max}}{q_{\min}},
\end{equation}
which would still yield a regret bound in $\Tilde{O}\left(N + \sqrt{NT}\right)$. This assumption of quasi-linear growth is reasonable and represents, for instance, a stationary state with fixed replenishing rates. It can be found in \cite{miao2021general}, from which the authors derive a $\Tilde{O}(\sqrt{NT} + NK^{5/2})$ regret bound. Assuming a quasi-linear growth of the inventory is equivalent to assuming that $q_{\min}^{-1}$ is bounded independently from $T$, up to polylog factors in $T$ and $N$ (i.e., $q^{-1}_{\min} = \Tilde{O}(1)$).

The power of our approach is that it extends beyond the large inventory assumption. The main limitation in current works is a regret bound that scales in $q_{\min}^{-1}$ (See \cite{Cheung_assort_inventory}, \cite{miao2021general}, where the authors base their regret bound on the size of $q_{\min}$). Following Inequality \eqref{eq:rinv}, we see that the impact of $r_{\sf inv}$ nuances the idea that the smallest inventory should be the main driver for its cost. We argue that it is not only the smallest inventory that affects the regret, but the \textit{overall saturation} of the capacity constraints in the optimal fluid regime. One gain from this improvement is that $r_{\sf inv}$ is well-behaved in the limiting cases, i.e., the cost of the inventory does not go to infinity as a product becomes rarer (as opposed to existing bounds in $q_{\min}^{-1}$). Moreover, due to the knapsack structure of LP$(\pi)$, in many cases $r_{\sf inv}$ has a simple structure, and can be bounded by above as:
\begin{equation}\label{eq:hqrnorm}
    r_{\sf inv} = \sum_{i \in [N]} r_i \frac{q_i^{\sf OPT}}{\sqrt{q_i}} \leq \sum_{i \in [N]} r_i \sqrt{q_i^{\sf OPT}} \leq N r_{\max} \sqrt{q_{\max}}.
\end{equation}

Hence, when $\boldsymbol{q}$ grows sub-linearly in $T$ (i.e., $q_{\min}^{-1} = \Theta(T^{\alpha})$ where $\alpha < 1$), we have from Inequality \ref{eq:hqrnorm}: $r_{\sf inv} \leq N r_{\max} T^{\frac{-\alpha}{2}}$, which when plugged in the regret bound in Theorem \ref{main_theorem} yields a regret bound of $\Tilde{O}\left(N\left(1 + T^{\frac{1 - \alpha}{2}}\right) + \sqrt{NT}\right)$. As a comparison with existing work in the literature, \cite{miao2021general} make the linear growth assumption $ q_{\min}^{-1} = O(1)$ for a $\Tilde{O}(\sqrt{NT} + N K^{5/2})$ bound, and do not provide a regret bound under the regime $\Theta(T^{\alpha})$. \cite{Cheung_assort_inventory} assume that $q_{\min}^{-1} \lesssim T^{1/3}$, from which they design a tractable algorithm for a regret bound of $\Tilde{O}(K^{2/3}N^{1/3}T^{2/3})$. They also design an (intractable) algorithm which assumes that $q_{\min}^{-1} \lesssim \sqrt{T}$, and yields a regret bound of $\Tilde{O}\left(K^{5/2} N \sqrt{T}\right)$. 
 
\subsection{Time complexity of Algorithm \ref{main_algo}}
We briefly discuss the time complexity of Algorithm \ref{main_algo}. In particular, we show that it is linear in $T$ and polynomial in $N$. Before the beginning of epoch $\ell \geq 1$, the seller solves the program $\text{LP$(\pi^{\sf UCB}_{\ell})$}$:
\begin{align*}
\text{LP$(\pi^{\sf UCB}_{\ell-1})$}:\qquad \text{max }& \sumS y_S \sum_{i \in [N]} r_i \pi^{\sf UCB}_{\ell-1}(i, S) \\
    \text{s.t }&\sumS y_S \pi^{\sf UCB}_{\ell-1}(i, S) \leq (1 - \omega_{i, \ell-1}) q_i \; \; \; \forall i = 1, \ldots, N \\ &y \geq 0 \; \;, \; \sumS y_S \leq 1.
\end{align*}
For any exact solver $\sf A$ of $\text{LP$(\pi^{\sf UCB}_{\ell})$}$, we denote $\mathcal{C}_{\ell}^{\sf A}$ the time complexity for sampling $S_\ell \sim y_\ell$ solution to LP$(\pi^{\sf UCB}_\ell)$. 
% Notice that the size of the LP is only a function of $N$ and $K$. In particular, it does not depend on the time horizon $T$. 
and $\mathcal{C}^{\sf A} := \max_{\ell \geq 1} \mathcal{C}_{\ell}^{\sf A}$. Within each epoch $\ell$, Algorithm \ref{main_algo} updates its variables in constant time, hence:
\begin{equation}\label{eq:timecomplexity}
    \mathcal{C}^{\sf ALG} \lesssim \sum_{\ell \geq 1} \mathcal{C}_{\ell}^{\sf A} \leq \max_{\ell \geq 1} \mathcal{C}_{\ell}^{\sf A}\sum_{\ell \geq 1} 1 \leq T\mathcal{C}^{\sf A}.
\end{equation}
Since the size of the program P$(\pi^{\sf UCB}_\ell)$ is only a function of $N$ and $\mathcal{S}$, so is $\mathcal{C}^{\sf A}$. In particular, $\mathcal{C}^{\sf A}$ is not a function of time $T$, and Equation \eqref{eq:timecomplexity} guarantees that Algorithm \ref{main_algo} is at most linear in $T$. By Assumption \ref{assumption:constraints}, $\mathcal{C}^{\sf A}$ is polynomial in $N$ (and $N$ only), therefore Algorithm \ref{main_algo} can be solved efficiently.

For most practical cases on $\mathcal{S}$ discussed in the previous section, $\text{LP$(\Tilde{\pi})$}$ has an exact separation oracle. Consequently, producing efficiently an optimal solution with known linear support in $N$ for \text{LP$(\Tilde{\pi})$} has already been studied (See \cite{liu2008choice}). The ellipsoid method guarantees a polynomial time in $N$ for all the problem instances, while in practice, column generation \cite{desaulniers2006column} proves to be empirically more efficient. 

Finally, it is worthy to note that $\mathcal{C}^{\sf A}$ does not necessarily increase with $K = \max_{S \in \mathcal{S}} |S|$. Even though it is linear in $N$ in subsets of small supports (i.e., assortments of $1$ element), it can also be linear in larger subsets of assortments: in the special case of no constraints $\mathcal{S} = [N]$, \cite{elmachtoub2022revenue} achieve a linear algorithm in $N$. This justifies the more general Assumption \ref{assumption:constraints}, that essentially decorrelates the computational efficiency of the LP relaxation from the learning aspect of the online problem.

\section{Regret analysis}
\label{sec:regret_analysis}
We devote the rest of the paper to expose the main ideas behind the proof of Theorem \ref{main_theorem}, leaving the details in the Appendix. 

Deriving the regret bound requires three major steps. First, we establish the structural properties of the estimators $\boldsymbol{\Tilde{v}}_\ell$, most of which derive from \cite{DBLP:journals/corr/AgrawalAGZ17a}. From there we will construct high confidence bounds for the consumption rate $\pi$. Second, we show that with high probability, none of the products is depleted before the end of the time horizon. This will be done by analyzing the consumption of a fictional algorithm that does not stop when a product is depleted, and is allowed to exceed the inventory capacity. The consumption of product $i$ at the end of time $T$ for such an algorithm will be denoted $\mathcal{I}_{i, T}$, and we will prove that with high probability that for each product $i$, $\mathcal{I}_{i, T} \leq T q_i $. The main consequence is that Algorithm \ref{main_algo} has a stopping time of exactly $T$ (with high probability). This is formalized in Proposition \ref{proposition:consumption}. Third, we show that conditionally on this high probability event, the sequence of assortments successively offered by MNLwK-UCB generates a high enough revenue. 

At the heart of these three steps lies a refined analysis of $\mathcal{I}_{i, T}$. In particular, we decompose the exact inventory consumption $\mathcal{I}_{i, T}$ into a main term derived from Algorithm \ref{main_algo} and three independent noise terms. The first noise term is the cost of the seller offering a randomized assortment, the second noise term is the cost of the customer making a randomized purchase, and the third term is the mis-specification cost in the offered assortments. The first two noise terms are results of both players (the seller and the customer) randomizing their decisions in a discrete repeated game, and can be analyzed without leveraging many structural properties of Algorithm \ref{main_algo}. The third noise term captures the lack of knowledge of $\boldsymbol{v}$ and the learning aspect of the problem. This term will consequentely be the main driver of the regret. Naturally, Algorithm \ref{proposition:consumption} is tuned so that this term is minimal.

Before delving into the details of each step, we would like to remind the reader of some essential properties of the epoch lengths $\{|\mathcal{E}_\ell|\}_{\ell \geq 1}$ and the estimators $\{\tilde{v}_{i,\ell}\}_{i \in [N], \ell \geq 1}$. These statements are already established in \cite{DBLP:journals/corr/AgrawalAGZ17a}.

\begin{lemma}[\cite{DBLP:journals/corr/AgrawalAGZ17a}]
\label{lemma:essential_properties}
The following properties hold:
\begin{enumerate}[label=(\roman*)]
    \item Conditionally on $S_{\ell}$, the length of the $\ell^{\sf th}$ epoch $|\epochl|$ is a geometric random variable with parameter $\frac{1}{1+\sum_{i\in S_{\ell}}v_i}$.
    \item Conditionally on both $S_\ell$ and $|\mathcal{E}_\ell|$, $n_i(\ell)$ is a Binomial random variable with parameter $(|\mathcal{E}_\ell|, \pi(i, S_\ell|\boldsymbol{v}))$.
    \item For every item $i$ and epoch $\ell$, the estimator $\tilde{v}_{i,\ell}$ is unbiased.

    % \item For every item $i$ and epoch $\ell$, $v_i\in [\vlcb_{i,\ell}, \vucb_{i,\ell}]$ with probability at least $1- \tfrac{6}{N\ell}$.
    % \item There exist some universal constants $C_1$ and $C_2$ such that for all $i\in [N]$, $\ell\in [L]$, 
    % \begin{equation*}
    %    \vucb_{i, \ell} - \vlcb_{i, \ell} \leq C_1 \sqrt{\frac{v_i\log{({\sqrt{N}}T^2+1)}bis}{T_i(\ell)}} + C_2\left(\frac{\log{({\sqrt{N}}T^2+1)}bis}{T_i(\ell)}\right),
    % \end{equation*}
    % with probability at least $1-7/N\ell$.
\end{enumerate}
\end{lemma}
%The estimators $\{\bar{v}_{i,\ell}\}$ and confidence bounds $\{\vucb_{i,\ell}\}$ and $\{\vlcb_{i,\ell}\}$ satisfy:
\subsection{High probability confidence bounds for $v$ and $\pi$}
We remind the reader of the expressions of $\{\vucb_{i,\ell}\}$, $\{\vlcb_{i,\ell}\}$ for both $v$ and $\Tilde{v}_\ell$, which construction follows the distributional properties established in Lemma \ref{lemma:essential_properties}:
\begin{equation*}
    \vucb_{i,\ell} = \min\left\{\bar{v}_{i,\ell} + \Bigg( \displaystyle \sqrt{\bar{v}_{i,\ell}\frac{48\log{({\sqrt{N}}T^2+1)}}{T_i(\ell)}} + \frac{48\log{({\sqrt{N}}T^2+1)}}{T_i(\ell)} \Bigg),1\right\},
\end{equation*}
\begin{equation*}
\vlcb_{i,\ell} = \max\left\{\bar{v}_{i,\ell} - \Bigg( \displaystyle \sqrt{\bar{v}_{i,\ell}\frac{48\log{({\sqrt{N}}T^2+1)}}{T_i(\ell)}} + \frac{48\log{({\sqrt{N}}T^2+1)}}{T_i(\ell)} \Bigg), 0\right\}.
\end{equation*}
The high probability bounding $v \in [\vlcb, \vucb]$ is formally stated in the following lemma:
\begin{lemma}[Confidence bounds for $\boldsymbol{v}$]\label{lemma:vcbw}
    The following properties hold
    \begin{enumerate}[label=(\roman*)]
    \item For every item $i$ and epoch $\ell$, $v_i\in [\vlcb_{i,\ell}, \vucb_{i,\ell}]$ with probability at least $1- \tfrac{6}{NT^2}$.
    \item There exist some universal constants $C_1$ and $C_2$ such that for all $i\in [N]$, $\ell\in [L]$, 
    \begin{equation*}
       \vucb_{i, \ell} - \vlcb_{i, \ell} \leq C_1 \sqrt{\frac{v_i\log{({\sqrt{N}}T^2+1)}}{T_i(\ell)}} + C_2\left(\frac{\log{({\sqrt{N}}T^2+1)}}{T_i(\ell)}\right),
    \end{equation*}
    with probability at least $1-\frac{7}{NT^2}$.
\end{enumerate}
\end{lemma}
While property $(i)$ shows that $\vucb_{i, \ell}$ and $\vlcb_{i, \ell}$ are indeed confidence bounds for the true parameters $v_i$, property $(ii)$ shows that we can bound the width of this confidence interval with the true parameters. To restate property $(ii)$ simply, we have
\begin{equation*}
   \vucb_{i, \ell} - \vlcb_{i, \ell} \lesssim 
    \frac{1}{\sqrt{T_i(\ell)}}
\end{equation*}
with high probability (where $\lesssim$ hides numerical constants, dependencies in $\boldsymbol{v}$, and $\log$ dependencies in $T, N$). Our regret analysis relies on the stronger property given in Lemma \ref{lemma:vcb}, which implies that the properties (i) and (ii) are simultaneously satisfied by all the estimators, for a smaller price in probability. Unlike in \cite{DBLP:journals/corr/AgrawalAGZ17a}, this stronger property is required to analyse the total product consumption and show that the hard inventory constraints are respected with high probability. The proof of Lemma \ref{lemma:vcb} is stated in Appendix \ref{app:vcb}.
\begin{lemma}\label{lemma:vcb}
For each $i \in [N]$ and $\ell \in [L]$, we introduce the following events:
\small{
\begin{align*}
    \mathcal{A}^{\sf CB1}_{i, \ell} &= \bigg\{v_i \in \left[\vlcb_{i, \ell},\vucb_{i, \ell}\right]\bigg\}, \mathcal{A}^{\sf CB1} = \bigcap_{i \in [N]} \bigcap_{\ell \geq 1}  \mathcal{A}^{\sf CB1}_{i, \ell-1},
    \\
    \mathcal{A}^{\sf CB2}_{i, \ell} &:= \Bigg\{\left|\vucb_{i, \ell} - v_i\right| + \left|v_i - \vlcb_{i, \ell}\right| \leq C_1 \sqrt{\frac{v_i\log{({\sqrt{N}}T^2+1)}}{T_i(\ell)}} + C_2\Big(\frac{\log{({\sqrt{N}}T^2+1)}}{T_i(\ell)}\Big) \Bigg\},  \mathcal{A}^{\sf CB2} = \bigcap_{i \in [N]} \bigcap_{\ell \geq 1}  \mathcal{A}^{\sf CB2}_{i, \ell-1}.
\end{align*}
}
We have $\prob_\pi(\mathcal{A}^{\sf CB1} \cap \mathcal{A}^{\sf CB2}) \geq 1 - \frac{13}{T}$.
\end{lemma}
Following Lemma \ref{lemma:vcb}, it is easy to see that conditionally on $\mathcal{A}^{\sf CB1}$, $\pi^{\sf UCB} \geq \pi$:
\begin{equation*}
    \pi^{\sf UCB}_{\ell}(i, S)= \frac{\vucb_{i, \ell}}{1 + \sum_{j \in S}\vlcb_{i, \ell}}\mathbbm{1}_{i \in S} \geq \frac{v_i}{1 + \sum_{j \in S}v_j}\mathbbm{1}_{i \in S} = \pi(i, S),
\end{equation*}
where we simultaneously use $\vucb_{i, \ell} - v_i \geq 0$ and $v_i - \vlcb_{i, \ell} \geq 0$. We introduce the \textit{pointwise} estimation error for each product $i \in [N]$ and $\ell \geq 0$:
\begin{equation}\label{eq:epsilon}
    \epsilon_{i, \ell} := \left(|v_{i, \ell}^{\sf UCB} - v_i| + |v_i - \vlcb_{i, \ell}|\right)\mathbbm{1}_{i \in S_\ell}.
\end{equation}
$\boldsymbol{\epsilon}$ is the width of the confidence interval on item $i$ if offered in epoch $\ell$, and represents how well utilities are estimated. Because of the intricate nature of the MNL choice function $\pi$, we will see that the \textit{effective} estimation error
\begin{equation}\label{eq:w}
    w_{i, \ell} := \pi(i, S_\ell)\sum_{j \in S_\ell}\epsilon_{j, \ell} + \epsilon_{i, \ell}
\end{equation}
is just as relevant. In particular, we will show that with high probability, the noise terms are essentially a function of the four cumulative errors
\begin{equation*}
    \|\boldsymbol{\epsilon}_i\|_1 = \sum_{\ell \geq 1}\epsilon_{i, \ell}, \quad \|\boldsymbol{\epsilon}_{i}\|_2 := \sqrt{\sum_{\ell \geq 1} \epsilon_{i, \ell}^2}, \quad \|\boldsymbol{w}_{i}\|_1 := \sum_{\ell \geq 1} w_{i, \ell}, \quad \|\boldsymbol{w}_{i}\|_2 := \sqrt{\sum_{\ell \geq 1} w_{i, \ell}^2},
\end{equation*}
where $\|\cdot\|_1, \|\cdot\|_2$ are respectively the $1$-norm and the $2$-norm.

\subsection{Bounding the inventory consumption}

Let $\tau$ be the (random) stopping time of Algorithm \ref{main_algo}. Because the algorithm stops as soon as a product is depleted, we must have $\tau \leq T$ (and as a consequence, MNLwK-UCB is feasible). In this section, we analyse the inventory consumption of a fictive algorithm that does not follow this stopping rule. Up until $\tau$, the two algorithms behave exactly the same. However, nothing prohibits the fictive algorithm from exceeding the inventory constraints. The main goal of this section is to prove that this does not happen with high probability. To be precise, we fix a product $i \in [N]$ and we focus on the study of $\mathcal{I}_{i, T}$, the exact consumption of product $i$ at the end of the time horizon $T$ for this fictive algorithm. The main result of this section is that $\mathcal{I}_{i, T}$ never exceeds the available inventory:
\begin{proposition}\label{proposition:consumption}
    The inequalities $\mathcal{I}_{i, T} \leq q_i T$ hold simultaneously for all products $i \in [N]$ with probability at least $1 - \frac{15}{T}$. Or equivalently, Algorithm \ref{main_algo} stops exactly at $T$ with probability at least $1 - \frac{15}{T}$, i.e., $\tau = T$.
\end{proposition}

The proof of Proposition \ref{proposition:consumption} will be done through successive decompositions that translate the different aspects of the interaction between the seller and the consumer. First, notice that $\mathcal{I}_{i,T}$ can be decomposed by looking into each (random) epoch $\mathcal{E}_\ell$ separately:
\begin{equation*}
    \mathcal{I}_{i, T} = \sum_{t \geq 1} \mathbbm{1}_{c_t = i} = \sum_{\ell \geq 1} \sum_{t \in \mathcal{E}_\ell}\mathbbm{1}_{c_t = i} = \sum_{\ell \geq 1} n_i(\ell),
\end{equation*}
where $n_i(\ell) = \sum_{t \in (\ell)}\mathbbm{1}_{c_t = i}$ is the total consumption of product $i$ in epoch $\mathcal{E}_\ell$.
For each epoch $\ell \geq 1$, and conditionally on the played assortment $S_\ell$ and the size of the epoch $|\mathcal{E}_\ell|$, each $\mathbbm{1}_{c_t = i}$ is an independent Bernoulli random variable of parameter $\pi(i, S_\ell)$, so that
\begin{equation*}
    \forall \ell \geq 1, \quad n_i(\ell)|S_\ell, |\mathcal{E}_\ell| \sim \text{Bin}(|\mathcal{E}_\ell|, \pi(i, S_\ell|\boldsymbol{v})).
\end{equation*}
This is due to the inherently random choice function $\pi$, and can be seen as an effect of the customer's random behavior.

The seller acts as if $(\pi^{\sf UCB}_{\ell})_{\ell \geq 1}$ are the true choice functions $\pi$, and consequently suffers a mis-specification error $\pi(i, \cdot) - \pi^{\sf UCB}_{\ell}(i, \cdot)$:
\begin{equation*}
    \mathcal{I}_{i, T} = \sum_{\ell \geq 1} |\mathcal{E}_\ell| \pi^{\sf UCB}_{\ell-1}(i, S_\ell) + \sum_{\ell \geq 1}\left( n_i(\ell) - |\mathcal{E}_\ell| \pi(i, S_\ell)\right) + \sum_{\ell \geq 1} |\mathcal{E}_\ell|\left( \pi(i, S_\ell) - \pi^{\sf UCB}_{\ell-1}(i, S_\ell)\right).
\end{equation*}
Next, we integrate how the offered assortment $S_{\ell}$ is chosen. In our case, the seller randomizes over $\mathbb{P}(S_\ell = S)=y_{\ell}(S)$, so that
\begin{equation*}
    \sum_{\ell \geq 1}|\mathcal{E}_\ell|\pi^{\sf UCB}_{\ell-1}(i, S_\ell) = \sum_{\ell \geq 1}|\mathcal{E}_\ell| \sum_{S \in \mathcal{S}}y_\ell(S)\pi^{\sf UCB}_{\ell-1}(i, S) + \sum_{\ell \geq 1}|\mathcal{E}_{\ell}|\left(\pi^{\sf UCB}_{\ell-1}(i, S_\ell)- \sum_{S \in \mathcal{S}}y_\ell(S)\pi^{\sf UCB}_{\ell-1}(i, S)\right).
\end{equation*}
This gives us the following inventory consumption decomposition:
\begin{equation}\label{eq:inventorydecomposition}
    \mathcal{I}_{i, T} = \sum_{\ell \geq 1}|\mathcal{E}_\ell| \sum_{S \in \mathcal{S}}y_\ell(S)\pi^{\sf UCB}_{\ell-1}(i, S) + \delta_{i, T}^{\sf random} + \delta_{i, T}^{\sf mnl} + \delta_{i, T}^{\sf shift}.
\end{equation}
where $\sum_{S \in \mathcal{S}}y_\ell(S)\pi^{\sf UCB}_{\ell-1}(i, S)$ is the consumption the seller was hoping to achieve (and is bounded by $(1 - \omega_i)q_i$), and where
\begin{align*}
    \delta_{i, T}^{\sf random} &:= \sum_{\ell \geq 1}|\mathcal{E}_{\ell}|\left(\pi^{\sf UCB}_{\ell-1}(i, S_\ell)- \sum_{S \in \mathcal{S}}y_\ell(S)\pi^{\sf UCB}_{\ell-1}(i, S)\right),\\
    \delta_{i, T}^{\sf mnl} &:= \sum_{\ell \geq 1}\left( n_i(\ell) - |\mathcal{E}_\ell| \pi(i, S_\ell)\right),\\
    \delta_{i, T}^{\sf shift}&:=\sum_{\ell \geq 1} |\mathcal{E}_\ell|\left( \pi(i, S_\ell) - \pi^{\sf UCB}_{\ell-1}(i, S_\ell)\right).
\end{align*}
\noindent Equation \eqref{eq:inventorydecomposition} states that up to three noise terms $\delta_{i, T}^{\sf random}, \delta_{i, T}^{\sf mnl}$, and $\delta_{i, T}^{\sf shift}$, the inventory consumption of product $i$ at the end of time $T$ is the same as the first term $ \sum_{\ell \geq 1}|\mathcal{E}_\ell| \sum_{S \in \mathcal{S}}y_\ell(S)\pi^{\sf UCB}_{\ell-1}(i, S)$, a quantity over which the seller has full control through \text{LP$(\pi^{\sf UCB}_\ell)$}. The first noise term $\delta_{i, T}^{\sf random}$ measures the cost of randomizing over the set of assortments, or in other words, the cost of viewing the consumption in the fluid model. The second noise term $\delta_{i, T}^{\sf mnl}$ measures the cost of randomization  incurred by the customer. This depends uniquely on the inherently random customer behavior, which is in our particular case, a MNL-choice behavior. Both of these terms can be seen as a consequence of randomizing strategies between the two players (the seller and the customer) in a repeated game. It is easy to see that over time, the effect of this randomizing becomes negligible, since on the one hand the seller will have many opportunities to randomize (intuitively, there is a $\sim T$ number of epochs), and on the other hand, the customer's behavior remains the same through the selling horizon $T$. This is true for any sequence of decisions $\{y_\ell\}_{\ell \geq 1}$, and any choice model $\pi$, as we can derive a $\tilde{O}(\sqrt{T})$ bound without using the optimality of $y_\ell$ or the MNL structure of $\pi$.

The third noise term $\delta_{i, T}^{\sf shift}$ measures the cumulative cost of mis-specifying the true model $\pi(i, \cdot)$ for $\pi^{\sf UCB}_{\ell}(i, \cdot)$. This is where the quality of the seller's optimism $\{\pi^{\sf UCB}_{\ell}\}_{\ell \geq 1}$ and the stopping rule of each epoch $\ell$ simultaneously intervene. 

The next step of the proof will consist of bounding each noise term separately. The first two noise terms $\delta_{i, T}^{\sf random}$ and $\delta_{i, T}^{\sf mnl}$ can be viewed as centered bounded increments, which yields a high probability $\tilde{O}(\sqrt{T})$ bound by Azuma-Hoeffding. The implication of this bound is that, regardless of the seller's decisions and the customer's behavior, one can always view this problem in the fluid regime without incurring a high regret in $T$. However, this bounding is loose in the other parameters. In particular, it does not use the scalability of $\boldsymbol{v}$ in $\pi$, or that the seller makes increasingly accurate estimates of the true distribution $\pi$. 

The major shortcoming of the previous concentration bounds is that they use a global bounding of $\boldsymbol{\epsilon}_i$. One way to overcome this shortcoming is to account for each local error, and instead of using a global bounding, uses an epoch-dependent bounding. This technique is common and can be found in \cite{Agrawal_global_bandits}, \cite{Kleinberg2008} and \cite{Babaioff}. We state the result formally in the following Lemma:
\begin{lemma}[Bounding $\delta_{i, T}^{\sf random}$]\label{lem:deltarandommnl}
    There exists a universal constant $C_0$ such that the following inequality
    \begin{equation*}
        \left|\delta_{i, T}^{\sf random}\right| \leq 3\sqrt{C_0 \log T}\sqrt{\sum_{\ell \geq 1} |\mathcal{E}_\ell| \sum_{S \in \mathcal{S}}y_\ell(S)\pi^{\sf UCB}_{\ell-1}(i, S)}+ 3C_0 \log T .
    \end{equation*}
    holds with probability at least $1 - \frac{1}{T}$. Such an event is denoted $\mathcal{A}^{\sf random}$.
\end{lemma}
The same analysis can be applied to $\left|\delta_{i, T}^{\sf mnl}\right|$, where we obtain with probability at least $1 - \frac{1}{T}$,
\begin{equation}\label{eq:concentrationmnl}
    \left|\delta_{i, T}^{\sf mnl}\right| \leq 3\sqrt{C_0 \log T}\sqrt{\sum_{\ell \geq 1} |\mathcal{E}_\ell| \pi(i, S_\ell)}+ 3C_0 \log T.
\end{equation}
Inequality \eqref{eq:concentrationmnl} can be seen as a quadratic inequality in $\sqrt{\sum_{\ell \geq 1} |\mathcal{E}_\ell| \pi(i, S_\ell)}$, over which the seller has less control. A better more suitable bounding, derived from Inequality \eqref{eq:concentrationmnl}, is formally stated in the following Lemma:
\begin{lemma}[Bounding $\delta_{i, T}^{\sf mnl}$]\label{lem:deltarandommnl2} The inequality $\delta_{i, T}^{\sf mnl} \leq 3\sqrt{C_0 \log T}\sqrt{\mathcal{I}_{i, T}} + 3 C_0 \log T$ holds with probability at least $1 - \frac{1}{T}$. Such an event is denoted $\mathcal{A}^{\sf mnl}$.
\end{lemma}
We now deal with $\delta_{i, T}^{\sf shift}$. By choice of MNLwK-UCB, we have conditionally on $\mathcal{A}^{\sf CB1}$, $\pi \leq \pi^{\sf UCB}$, so that $\delta_{i, T}^{\sf shift} \leq 0$. Intuitively, choosing with respect to an optimistic choice function $\pi^{\sf UCB}$ makes the consumption safer, regardless of the choice of the confidence bound. However, the price the seller pays for this optimism is deferred to the accumulated revenue, which study is conducted in the next section.

We introduce the following high probability event:
\small
\begin{equation*}
    \mathcal{A} := \mathcal{A}^{\sf random} \cap \mathcal{A}^{\sf mnl} \cap \mathcal{A}^{\sf CB1}\cap \mathcal{A}^{\sf CB2}. \qquad \text{By union bound:} \qquad \mathbb{P}\left(\mathcal{A}\right) \geq 1 - \frac{1 + 1 + 13}{T} = 1 - \frac{15}{T}.
\end{equation*}
\normalsize
Applying in the following order Lemmas \ref{lem:deltarandommnl}, \ref{lem:deltarandommnl2}, and $\delta_{i, T}^{\sf shift} \leq 0$ gives conditionally on $\mathcal{A}$:
\begin{equation}\label{eq:shift}
    \delta_{i, T}^{\sf random} + \delta_{i, T}^{\sf mnl} + \delta_{i, T}^{\sf shift} \leq A_{q, 0} + A_{q,2} \sqrt{\sum_{\ell \geq 1} |\mathcal{E}_\ell| \sum_{S \in \mathcal{S}}y_\ell(S)\pi^{\sf UCB}_{\ell-1}(i, S)} + A_{q, 1}\sqrt{\mathcal{I}_{i, T}},
\end{equation}
where $A_{q, 0} = 6 C_0 \log  T$, $A_{q, 1} = A_{q, 2} = 3 \sqrt{C_0 \log T}$ are numerical constants times polylog factors in $T$. From Equation \ref{eq:inventorydecomposition} and \eqref{eq:shift}, we have:
\begin{equation}\label{eq:systeminventory}
    \mathcal{I}_{i, T} - A_{q, 1}\sqrt{\mathcal{I}_{i, T}} - \sum_{\ell \geq 1}|\mathcal{E}_\ell| \sum_{S \in \mathcal{S}}y_\ell(S)\pi^{\sf UCB}_{\ell-1}(i, S) -  A_{q, 2} \sqrt{\sum_{\ell \geq 1} |\mathcal{E}_\ell| \sum_{S \in \mathcal{S}}y_\ell(S)\pi^{\sf UCB}_{\ell-1}(i, S)} - A_{q, 0} \leq 0.
\end{equation}
We show that Inequality \eqref{eq:systeminventory} is sufficient to conclude that the inventory is never depleted before $T$ with high probability:
\begin{lemma}\label{lemma:algebraic}
    Inequality \eqref{eq:systeminventory} implies $\mathcal{I}_{i, T} \leq q_i T$.
\end{lemma}

The proof of Lemma \ref{lemma:algebraic} is technical and is deferred to Appendix \ref{app:lemmas}. It requires solving successively two quadratic systems. A first one to derive an upper bound on $\sqrt{\mathcal{I}_{i, T}}$, and a second one to derive an sufficient condition on $\sum_{\ell \geq 1}|\mathcal{E}_\ell| \sum_{S \in \mathcal{S}}y_\ell(S)\pi^{\sf UCB}_{\ell-1}(i, S)$ for which we have $\mathcal{I}_{i, T}$. We then show that this sufficient condition is satisfied under the inequality $\sum_{\ell \geq 1}|\mathcal{E}_\ell| \sum_{S \in \mathcal{S}}y_\ell(S)\pi^{\sf UCB}_{\ell-1}(i, S) \leq (1 - \omega_i)q_i$, which stems from the constraints of the program LP$(\pi^{\sf UCB}_{\ell - 1})$, provided the expression of $\omega_i$.

Lemma \ref{lemma:algebraic} proves that with high probability, no products are depleted before the end of the time horizon, or equivalently, the stopping time of Algorithm \ref{main_algo} is $T$. This completes the proof of Proposition \ref{proposition:consumption}.

\noindent \textbf{Bounding the accumulated mis-specification:}

For completeness of the study, we bound the magnitude of $\delta_{i, T}^{\sf shift}$. Bounding $\delta_{i, T}^{\sf shift}$ from below requires using the locally Lipschitz property satisfied by MNL choice models, which we formally state in the following lemma:
\begin{lemma}\label{lem:lipshitz}
    Conditionally on $\mathcal{A}^{\sf CB1}$, we have for each $i \in S \subset [N]$, and $\ell \geq 0$,
    \begin{equation*}
        \left|\pi(i, S) - \pi^{\sf UCB}_{\ell}(i, S)\right| \leq \frac{(\vucb_{i, \ell} - v_i) + \pi(i, S)\sum_{j \in S} v_j - \vlcb_{j, \ell}}{1 + V(S)}.
    \end{equation*}
\end{lemma}
Lemma \ref{lem:lipshitz} implies the following bounding:

\begin{corollary}\label{eq:shifterror}
Conditionally on $\mathcal{A}^{\sf CB1}$,
	\begin{equation*}
	   \left|\delta_{i, T}^{\sf shift}\right|
     \leq \left|\sum_{\ell \geq 1}|\mathcal{E}_\ell|(\pi(i, S_\ell) - \pi^{\sf UCB}_{\ell-1}(i, S_\ell))\right| \leq\sum_{\ell \geq 1}\frac{|\mathcal{E}_\ell|}{1 + V(S_\ell)}w_{i, \ell-1}.
	\end{equation*}
	where $w$ is defined in \eqref{eq:w}.
\end{corollary}

\noindent Following Corollary \ref{eq:shifterror}, it remains to bound the \textit{weighted} error
\begin{equation*}
	\sum_{\ell \geq 1} \frac{|\mathcal{E}_\ell|}{1 + V(S_\ell)}w_{i, \ell-1}.
\end{equation*}

The previous sums can be seen a summation over $\{w_{i, \ell-1}\}_{\ell \geq 1}$, re-weighted with the (random) weights $\left\{\frac{|\mathcal{E}_\ell|}{1 + V(S_\ell)}\right\}_{\ell \geq 1}$. Lemma \ref{lemma:essential_properties} states that each weight has mean $1$, therefore one can hope that for a large number of epochs, the effect of this re-weighting is neutral, i.e., one can hope for an approximation of the type:
\begin{equation*}
\sum_{\ell \geq 1} \frac{|\mathcal{E}_\ell|}{1 + V(S_\ell)}w_{i, \ell-1} \approx \sum_{\ell \geq 1} w_{i, \ell-1},
\end{equation*}
where the approximation loss of each term is small, and where $T_i$ is the number of times product $i$ has been offered at the end of time $T$. Classic concentration inequalities do not apply here, mainly because the weights $\left\{\frac{|\mathcal{E}_\ell|}{1 + V(S_\ell)}\right\}_{\ell \geq 1}$ are unbounded. Most concentration inequalities where the variables are unbounded require a uniform upper bound on the first moments. In our case, these inequalities would still not be enough for a satisfying regret bound, as the means of $|\mathcal{E}_\ell|$ are bounded by $V(S_\ell) + 1$, which can be as high as $K = \max_{S \in \mathcal{S}} |S| \leq N$, and this would yield a regret bound with sub-optimal dependency in $ N$. Without any additional use of the algorithm's structure, this issue cannot be overcome since even in the simplest case where $\{\mathcal{E}_\ell\}_{\ell \geq 1}$ are assumed to be independent, the sharpest known bound on the sum yields a bound that scales in $K^2$. If one wants to obtain a sharper bound, it is necessary to factor in the geometric structure of $|\mathcal{E}_\ell|$, as well as the variations within the sum. This is expressed in the following high probability (intermediate) bounding:

\begin{lemma}\label{lemma:geometricconcentration}
The inequality $\sum_{\ell \geq 1}\frac{|\mathcal{E}_\ell|}{1 + V(S_\ell)}w_{i, \ell-1} \leq 2\log T\|\boldsymbol{w}_i\|_2 + \|\boldsymbol{w}_i\|_1$ holds with probability at least $1 - \frac{8}{T}$. We denote such an event $\mathcal{A}^{\sf geom}$.
\end{lemma}

We use arguments from the proof of Freedman's inequality (See Theorem 2.1 in \cite{fan2015exponential}) combined with optimal sampling. The complete proof is stated in Appendix \ref{app:lemmas}.
\noindent It remains to bound the errors $\|\boldsymbol{w}_i\|_1$ and $\|\boldsymbol{w}_i\|_2$. This is where the accuracy of our confidence bounds intervenes, which is captured in the event $\mathcal{A}^{\sf CB2}$ introduced in Lemma \ref{lemma:vcb}:
\begin{lemma}\label{lemma:errors}
    For $\ell \geq 0$ and $i \in [N]$, and let $N_i$ be the total number of times product $i$ is offered up to time $T$. We have conditionally on $\mathcal{A}^{\sf CB2}$,
    \begin{align*}
        \|\boldsymbol{\epsilon}_i\|_1 &\leq 2\Tilde{C}_1 v_i\sqrt{\log\left(\sqrt{N}T^2 + 1\right)N_i} + (\Tilde{C}_2v_i + \Tilde{C}_3)\log^2\left(\sqrt{N}T^2 + 1\right),\\
        \|\boldsymbol{\epsilon}_i\|_2 &\leq \sqrt{2C_1^2 + 4C_2^2}\log\left(\sqrt{N}T^2 + 1\right),\\
        \|\boldsymbol{w}_i\|_1 &\leq 2\Tilde{C}_1 v_i \|\boldsymbol{v}\|_2\sqrt{\log\left(\sqrt{N}T^2 + 1\right)T} + \log^2\left(\sqrt{N}T^2 + 1\right)\left(2\Tilde{C}_2 v_i\|\boldsymbol{v}\|_1 + 2\Tilde{C}_3 N\right),\\
        \|\boldsymbol{w}_i\|_2 &\leq 2\sqrt{2}(C_1+ C_2\sqrt{2})\log\left(N T^2 + 1\right)\sqrt{KN }.
    \end{align*}
\end{lemma}

Combining Corollary \ref{eq:shifterror}, Lemma \ref{lemma:geometricconcentration}, and Lemma \ref{lemma:errors} yields the following high probability bounding on the shift error:
\begin{equation*}
    \left|\delta_{i, T}^{\sf shift}\right| \lesssim v_i \|\boldsymbol{v}\|_2 \sqrt{T} + N.
\end{equation*}

\subsection{Decomposing the expected reward}
The third and last part of the regret analysis is to prove that conditionally on the non-depletion of the inventory (which we proved happens with high probability in the previous section), the generated reward is sufficiently high. We provide a sketch of how this is conducted, and defer the details to Appendix \ref{app:lemmas}. By using the same decomposition stated in Equation $\eqref{eq:inventorydecomposition}$, we can view the previous generated reward as being exactly:
\begin{equation*}
    \sum_{t \geq 1} r_{c_t} = \sum_{i \in [N]} r_i \sum_{\ell \geq 1}|\mathcal{E}_\ell|\sum_{S \in \mathcal{S}} y_\ell(S)|\pi^{\sf UCB}_{\ell-1}(i, S) + \sum_{i \in [N]}r_i(\delta_{i,  T}^{\sf random} + \delta_{i, T}^{\sf mnl} + \delta_{i,  T}^{\sf shift}).
\end{equation*}
The two noise terms $\delta_{i, T}^{\sf random}$ and $\delta_{i, T}^{\sf mnl}$ are centered. Therefore by taking the law of total expectation into the previous equation, we obtain
\begin{align*}
    \mathbb{E}\left[\sum_{t \geq 1} r_{c_t} \right] &= \mathbb{E}\left[\sum_{\ell \geq 1} |\mathcal{E}_\ell| \sum_{i \in [N]}r_i \sum_{S \in \mathcal{S}} y_\ell(S) \pi^{\sf UCB}_{\ell-1}(i,S)\bigg|\mathcal{A}^{\sf c}\right]\mathbb{P}(\mathcal{A}^{\sf c}) + \mathbb{E}\left[\sum_{i \in [N]} r_i \delta_{i, T}^{\sf shift}\bigg|\mathcal{A}^c\right]\mathbb{P}(\mathcal{A}^c) &(a)\\
    &+ \mathbb{E}\left[\sum_{\ell \geq 1} |\mathcal{E}_\ell| \sum_{i \in [N]}r_i \sum_{S \in \mathcal{S}} y_\ell(S) \pi^{\sf UCB}_{\ell-1}(i,S)\bigg|\mathcal{A}\right]\mathbb{P}(\mathcal{A}) + \mathbb{E}\left[\sum_{i \in [N]} r_i \delta_{i, T}^{\sf shift}\bigg|\mathcal{A}\right]\mathbb{P}(\mathcal{A}). &(b)
\end{align*}
$(a)$ is negligible as $\mathcal{A}^{\sf c}$ is a rare event $\left(\mathbb{P}(\mathcal{A}^c) \leq \frac{15}{T}\right)$, and
\begin{equation*}
    r_i \delta_{i, T}^{\sf shift} \geq - N r_{\max} T, \qquad
    \sum_{\ell \geq 1} |\mathcal{E}_\ell| \sum_{i \in [N]}r_i \sum_{S \in \mathcal{S}} y_\ell(S) \pi^{\sf UCB}_{\ell-1}(i,S) \geq 0,
\end{equation*}
so that
\begin{equation}\label{equation:terma}
    (a) \geq -\frac{15}{T}N r_{\max} T = -15 N r_{\max} \gtrsim -N r_{\max}.
\end{equation}
The first term in $(b)$ is the main contributor to the generated revenue is close to $T r_{\sf opt}$. Consequently, it is the main contributor to the general revenue: 
\begin{lemma}\label{lemma:firstterm}
    Conditionally on $\mathcal{A}$, we have for each epoch $\ell \geq 1$:
    \begin{equation*}
        \sum_{i \in [N]}r_i \sum_{S \in \mathcal{S}} y_\ell(S) \pi^{\sf UCB}_{\ell-1}(i,S) \geq r_{\sf opt} - r_{\sf disc}.
    \end{equation*} 
    where $r_{\sf disc} = \sum_{i \in [N]}r_i \omega_{i} \sum_{S \in \mathcal{S}} y^{\sf OPT}(S)\pi(i, S) = \sum_{i \in [N]}\omega_{i} r_i q^{\sf OPT}_i$.
\end{lemma}

Lemma \ref{lemma:firstterm} implies that conditionally on $\mathcal{A}$, each time-step generates the optimal revenue of the fluid relaxation $r_{\sf opt}$, up to a loss not exceeding $r_{\sf disc}$. The intuition behind $r_{\sf disc}$ is as follows. For each product $i$, $r_i q_i^{\sf OPT}$ represents the expected revenue from the sales of product $i$ in the optimal fluid LP regime. Since LP$(\pi^{\sf UCB}_{\ell -1}$) sacrifices $\omega_{i}$ of the inventory of product $i$, $\omega_{i} r_i q_i^{\sf OPT}$ represents the sacrificed revenue of product $i$, and $r_{\sf disc}$ represents the total sacrificed revenue over all products. This sacrificed revenue is accumulated over time, and we prove that conditionally on $\mathcal{A}$, the total sacrificed revenue satisfies
% Since in the previous section, we have proved that the algorithm runs exactly $T$ rounds (conditionally on $\mathcal{A}$), this implies
% \begin{equation}\label{equation:firstterm}
%     \mathbb{E}\left[\sum_{\ell \geq 1} |\mathcal{E}_\ell| \sum_{i \in [N]}r_i \sum_{S \in \mathcal{S}} y_\ell(S) \pi^{\sf UCB}_{\ell-1}(i,S)\bigg| \mathcal{A}\right] \geq T r_{\sf opt} -\sum_{\ell \geq 1} |\mathcal{E}_\ell| r_{\ell - 1}.
% \end{equation}
\normalsize
% where we applied $\vucb \leq 1$ and $\|\boldsymbol{v}^{\sf UCB}\|_2 \leq \sqrt{N}$ by Cauchy-Schwartz.  By setting $\Bar{\omega} = \frac{A_0 N + A_{1/2} \sqrt{\max_{S \in \mathcal{S}}|S|}\sqrt[4]{T} + A_1\sqrt{NT}}{T}$, we obtain
\begin{equation*}
    \sum_{\ell \geq 1}|\mathcal{E}_\ell| r_{\sf disc} = T \sum_{i \in [N]} \omega_i r_i q_i^{\sf OPT} \leq A_q \left(\sum_{i \in [N]} r_i \frac{q_i^{\sf OPT}}{\sqrt{q_i}}\right)\sqrt{T}  = A_q r_{\sf inv} \sqrt{T},
\end{equation*}
 \noindent where we choose $A_q = A_0 + A_1$. As a consequence of the equality above, $r_{\sf inv}$ will represent the first main driver of the $\sqrt{T}$ term in the regret. To deal with the second term of $(b)$, we once again apply the boundings derived in Lemma \ref{lemma:errors}, and we prove that conditionally on $\mathcal{A}$,
\begin{equation*}
   \sum_{i \in [N]} r_i \delta_{i, T}^{\sf shift} \gtrsim -\left(\|\boldsymbol{v}\|_2 r_{\sf opt} + \|\boldsymbol{v}\cdot\boldsymbol{r}\|_2\right) \sqrt{T} - r_{\max} N.
\end{equation*}
The previous term represents the loss in revenue due to the seller mis-specifying the model, and is expected to grow in $\sqrt{NT} + N$, in a similar fashion to its analog in the infinite inventory setting. In particular, this term will also be a main driver of the regret bound. Therefore
\small
\begin{equation}\label{eq:termb}
    (b) \gtrsim T r_{\sf opt} - \left(r_{\sf inv} + \|\boldsymbol{v}\|_2 r_{\sf opt} + \|\boldsymbol{v}\cdot\boldsymbol{r}\|_2\right)\sqrt{T} - r_{\max} N.
\end{equation}
\normalsize
\noindent By recalling that $\text{Regret}_T = T r_{\sf opt} -\mathbb{E}\left[\sum_{t \geq 1} r_{c_t} \right]$, combining inequalities \eqref{equation:terma} and \eqref{eq:termb} yields:
\small
\begin{equation*}
    \text{Regret}_T \lesssim N r_{max}+ \left(r_{\sf inv} + \|\boldsymbol{v}\|_2 r_{\sf opt} + \|\boldsymbol{v}\cdot\boldsymbol{r}\|_2\right)\sqrt{T}.
\end{equation*}
\normalsize
This provides the proof of of Theorem \ref{main_theorem}. 
\ACKNOWLEDGMENT{
This material is based upon work partially supported by: the National Science Foundation (NSF) grants CMMI 1636046, and the Amazon and Columbia Center of Artificial Intelligence (CAIT) PhD Fellowship.
% Enter the text of acknowledgments here
}% Leave this (end of acknowledgment)

% Bibliography
%\bibliographystyle{plain}

\bibliographystyle{ACM-Reference-Format}

\bibliography{sample-bibliography}
\newpage

% Appendix here
% Options are (1) APPENDIX (with or without general title) or 
%             (2) APPENDICES (if it has more than one unrelated sections)
% Outcomment the appropriate case if necessary
%

% Appendix
\newpage
\begin{APPENDIX}{}
\section{Concentration inequalities}
\subsection{Bernoulli Chernoff bounds}
We introduce the following well-known concentration theorems:

\begin{theorem}[Azuma-Hoeffding] \label{theorem:azuma_h}
Consider random variables $X_1, \ldots, X_n$, defined with respect to a filtration $\mathcal{F}_n$ and a stopping time $\tau \leq n$ a.s, and that $(X_k)$ is uniformly bounded by $X^*$. Then following inequality holds:

\begin{equation*}
    \prob\left(\sum_{k = \tau}^n X_k - \mathbb{E}(X_k|X_{k-1}, \ldots, X_1)> X^*\sqrt{2n \log T}\right) \leq \frac{1}{T}.
\end{equation*}
\end{theorem}
\noindent Note that Theorem \ref{theorem:azuma_h} is usually mentioned with a deterministic $\tau$. However, notice that for each realization of $\tau$, one can derive a sharper bound and relax it using $0 \leq \tau \leq n$. Since the probability does not depend on such a realization, a law of total probability gives the result.

\begin{theorem}\label{theorem:gen_chernoff}
Consider a sequence of random variables $X_1, \ldots, X_n \in [0, 1]$.

Set $\mu = \sum_{k = 1}^n \mathbb{E}(X_k|X_1, \ldots, X_{k-1})$ and $\tilde X = \sum_{k = 1}^n X_i$. There exists a universal constant $\lambda_0 > 0$ for which for all $\lambda > 0$, the following inequality holds:
\begin{equation*}
    \prob\Big(|\tilde X -  \mu| > 3 \sqrt{\lambda \mu n} + 3 \lambda \Big) \leq \exp(-\lambda/\lambda_0).
\end{equation*}
\end{theorem}
\noindent The proof of Theorem \ref{theorem:gen_chernoff} lies for instance in \cite{Babaioff}, and essentially combines optional sampling arguments with sharp Azuma-Hoeffding bounds. 

\medskip
\noindent \textbf{Proof of Lemma \ref{lem:deltarandommnl}}
\begin{proof}{Proof.}
    For each $t \geq 1$, let $\ell_t$ denote the unique (random) epoch index with $t \in \ell_t$. We have:
    \begin{equation*}
        \delta_{i, T}^{\sf random} = \sum_{t \geq 1} \left(\pi_{\ell_t - 1}^{\sf UCB}(i, S_{\ell_t}) - \sum_{S \in \mathcal{S}}y_{\ell_t}(S)\pi_{\ell_t - 1}^{\sf UCB}(i, S)\right),
    \end{equation*}
    where for each $t \geq 1$, $\pi_{\ell_t - 1}^{\sf UCB}(i, S_{\ell_t}) \in [0,1]$ and $\mathbb{E}\left[\pi_{\ell_t - 1}^{\sf UCB}(i, S_{\ell_t}) | y_{\ell_t}\right] = \sum_{S \in \mathcal{S}}y_{\ell_t}(S)\pi_{\ell_t - 1}^{\sf UCB}(i, S)$. Applying Theorem \ref{theorem:gen_chernoff} on the sequence $\{\pi_{\ell_t - 1}^{\sf UCB}(i, S_{\ell_t})\}_{t \geq 1}$ implies that
    \begin{equation*}
        \mathbb{P}\left(\left|\delta_{i, T}^{\sf random}\right| > 3 \sqrt{C_0  \log T \sum_{t \geq 1}\sum_{S \in \mathcal{S}} y_{\ell_t}(S)\pi_{\ell_t - 1}^{\sf UCB}(i, S)} + 3C_0  \log T\right) \leq \frac{1}{T},
    \end{equation*}
    which completes the proof of Lemma \ref{lem:deltarandommnl}.
    \hfill$\square$
\end{proof}
\noindent \textbf{Proof of Lemma \ref{lem:deltarandommnl2}}
\begin{proof}{Proof}
    By the same argument used in the proof of Lemma \ref{lem:deltarandommnl2}, the following inequality
    \begin{equation*}
        \mathbb{P}\left(\left|\delta_{i, T}^{\sf mnl}\right| > 3\sqrt{ C_0  \log T \left(\sum_{t \geq 1} \pi(i, S_{\ell_t}| \boldsymbol{v}) \right)} + 3C_0  \log T\right) \leq \frac{1}{T}.
    \end{equation*}
    holds with probability at least $1 - \frac{1}{T}$. Conditionally on this event,
    \begin{equation*}
        \left(\sum_{t \geq 1} \pi(i, S_{\ell_t}\right)-\mathcal{I}_{i, T} \geq - 3\sqrt{C_0 \log T} \sqrt{\left(\sum_{t \geq 1} \pi(i, S_{\ell_t}| \boldsymbol{v}) \right)} -  3C_0 \log T,
    \end{equation*}
    which is the same as $\sqrt{\sum_{t \geq 1} \pi(i, S_{\ell_t})}$ is the solution to the following (quadratic) system:
    \begin{equation*}
        X \geq 0, \qquad X^2 + A X - B \geq 0,
    \end{equation*}
        where $A = 3\sqrt{C_0 \log T}$ and $B = \mathcal{I}_{i, T} - 3 C_0 \log T$. If $X$ is a solution, $X$ must satisfy:
    \begin{equation*}
        X \geq \frac{1}{2}(\sqrt{A^2 + 4B} - A),
    \end{equation*}
    which implies
    \begin{equation*}
        X^2 \geq \frac{1}{4}\left(A^2 + 4B + A^2 - 2 A \sqrt{A^2 + 4B}\right) = B + \frac{A^2 - A\sqrt{A^2 + 4B}}{2},
    \end{equation*}
    or equivalently,
    \begin{equation*}
        B - X^2 \leq \frac{A\sqrt{A^2 + 4B} - A^2}{2} \leq A\sqrt{B},
    \end{equation*}
    Hence by using $\delta_{i, T}^{\sf mnl} = B - X^2$ and $A\sqrt{B} = \sqrt{\mathcal{I}_{i, T} - 3C_0 \log T} \leq 3\sqrt{C_0 \log T}\sqrt{\mathcal{I}_{i, T}}$, we obtain:
    \begin{equation*}
        \delta_{i, T}^{\sf MNL} \leq 3\sqrt{C_0 \log T}\sqrt{\mathcal{I}_{i, T}} + 3 C_0 \log T,
    \end{equation*}
    \noindent which completes the proof of Lemma \ref{lem:deltarandommnl2}. 
    \hfill$\square$
    \end{proof}
\subsection{Geometric Chernoff bounds}
\noindent \textbf{Proof of Lemma \ref{lemma:geometricconcentration}.}
\begin{proof}{Proof.}
We fix $i \in [N]$ and set $\Tilde{X}_\ell := \frac{|\mathcal{E}_\ell|}{1 + V(S_\ell)} - 1$. We have:
\begin{align*}
    \sum_{\ell \geq } \frac{|\mathcal{E}_\ell|}{1 + V(S_\ell)}w_{i, \ell} &=  \sum_{\ell \geq} w_{i,\ell}(\Tilde{X}_\ell + 1) \\
    &= \sum_{\ell \geq } w_{\ell - 1}^{i}\Tilde{X}_\ell + \sum_{\ell \geq } w_{i, \ell}.
\end{align*}
For each fixed $\ell \geq 1$, we denote $S_\ell^{\leq}$ the information space available at the realization of the assortment $S_\ell$. In which case $\{\tilde{X}_\ell\}_{\ell \geq 1}$ is $\{S_\ell^{\leq}\}_{\ell \geq 1}$-adapted and $\{w_{\ell-1}^i\}_{\ell \geq 1}$ is $\{S_\ell^{\leq}\}_{\ell \geq 1}$-predictible. is a filtration that is adapted and  we have for each $s \leq \min_{\ell \geq 1} w_{\ell-1}^i$:
\begin{align*}
    \mathbb{E}\left[\exp\left(s w_{\ell - 1}^{i}\Tilde{X}_\ell\right)\bigg|S_\ell^{\leq}\right] &= \exp\left(-s w_{\ell - 1}^{i}\right)\mathbb{E}\left[\exp\left(\frac{s  w_{\ell - 1}^{i}}{V(S_\ell) + 1}|\mathcal{E}_\ell|\right)\bigg| S_{\ell}^{\leq}\right] \\
    &= e^{-s w_{\ell - 1}^{i}} \times \frac{\frac{1}{1 + V(S_\ell)}}{e^{-\frac{s  w_{\ell - 1}^{i}}{V(S_\ell) + 1}} - 1+ \frac{1}{1 + V(S_\ell)}} \\
    &\leq \frac{e^{-s w_{\ell - 1}^{i}} \times \frac{1}{1 + V(S_\ell)}}{-\frac{s  w_{\ell - 1}^{i}}{V(S_\ell) + 1} + \frac{1}{1 + V(S_\ell)}} \\
    &= \frac{\exp(-s w_{\ell - 1}^{i})}{1 - s w_{\ell - 1}^{i}} \\
    &= \exp\left(-s w_{\ell}^i - \log(1- s w_{\ell - 1}^i)\right)\\
    &= \exp\left(\sum_{n \geq 2} s^n (w_{\ell - 1}^i)^n \right) \\
    &= \exp\left(\frac{s^2 (w_{\ell - 1}^i)^2}{1 - s w_{\ell - 1}^i}\right) \\
    &\leq \exp\left(2 s^2 (w_{\ell - 1}^i)^2 \right).
\end{align*}
where the first step stems from the $w_{\ell - 1}^i$ being $S_\ell^{\leq}$-measurable, the second step stems from the MGF of the geometric distribution applied to $|\mathcal{E}_\ell|$ conditionally on $S_\ell^{\leq}$, the third step stems from $e^x \geq x + 1$ for $x \in \mathbb{R}$, the fourth step stems from the simplification of $V(S_\ell) + 1$, and the later steps stem from Taylor expansions of $\log$ and geometric series with $s \leq 1$. Next, we have for $x > 0, s_1, \ldots, s_T \leq \frac{1}{\min_{\ell \geq 1} w_{\ell}^i}$. Since $w_{\ell - 1}^i \leq 1$, we have $\sum_\ell (w_{\ell - 1}^i)^2 \leq T$. Therefore
\begin{align*}
    \mathbb{P}\left(\sum_{\ell \geq 1} w_{\ell - 1}^i \Tilde{X}_\ell \geq x \sqrt{\sum_{\ell \geq 1} (w_{\ell - 1}^i)^2} \right) &\leq \sum_{t = 1}^{T}\mathbb{P}\left(\left(\sum_{\ell \geq 1} w_{\ell - 1}^i \Tilde{X}_\ell \geq x \sqrt{t}\right) \cap \left(t \geq \sum_{\ell \geq 1} (w_{\ell - 1}^i)^2 \right) \right) \\ &= \sum_{t = 1}^T \mathbb{P}\left(\left(\sum_{\ell \geq 1} s_t w_{\ell - 1}^i \Tilde{X}_\ell \geq s_t x \sqrt{t}\right)\cap \left(\sum_{\ell \geq 1} 2 s_t^2 (w_{\ell - 1}^i)^2 \leq 2 s_t^2 t\right) \right)\\
    &\leq \sum_{t = 1}^T \mathbb{P}\left(\sum_{\ell \geq 1}s_t w_{\ell - 1}^i \Tilde{X}_\ell + 2 s_t^2 t \geq s_t x \sqrt{t}+ \sum_{\ell \geq 1} (w_{\ell - 1}^i)^2 \right) \\
    &\leq \sum_{t = 1}^T\exp(-s_t x \sqrt{t} + 2 s_t^2 t )\mathbb{E}\left(\exp\left(\sum_{\ell \geq 1} s_t w_{\ell - 1}^i \Tilde{X}_\ell - 2s_t^2 (w_{\ell - 1}^i)^2\right)\right) \\
    &\leq \sum_{t = 1}^T \exp\left(-s_t x \sqrt{t} + 2 s_t^2 t\right) 
\end{align*}
We choose for each $t \geq 1$, $s_t = \frac{1}{\sqrt{t}} \leq \frac{1}{\min_{\ell \geq 1} w_{\ell - 1}^i}$. We have $-s_t x \sqrt{t} + 2s_t^2 t = - x + 2$. Therefore we have for $x > 0$,
\begin{equation*}
    \mathbb{P}\left(\sum_{\ell \geq 1} w_{\ell - 1}^i \Tilde{X}_\ell \geq x \sqrt{\sum_{\ell \geq 1} (w_{\ell - 1}^i)^2} \right) \leq T \exp\left(-(x - 2)\right) = \exp\left(\log T - (x - 2)\right).
\end{equation*}
For a better control of the right hand side, it is convenient to set $x = 2\log T$, so that

\begin{equation*}
    \mathbb{P}\left(\sum_{\ell \geq 1} w_{\ell - 1}^i \Tilde{X}_\ell \geq 2\log T\sqrt{\sum_{\ell \geq 1} (w_{\ell - 1}^i)^2} \right) \leq \frac{e^{-2}}{T} \leq \frac{8}{T},
\end{equation*}
which proves that first inequality of Lemma \ref{lemma:geometricconcentration} holds with probability at least $1 - \frac{8}{T}$. Deriving the second inequality follows the exact same steps, except that $(1 + V_\ell)\sum_{\ell \geq 1} \epsilon^2_{i, \ell-1} \leq \left(K\right)^2 T\leq N^2 T$. By union bound, the two inequalities hold simultaneously with probability at least $1 - \frac{16}{T}$, which completes the proof of Lemma \ref{lemma:geometricconcentration}.
\hfill$\square$
\end{proof}
\subsection{Properties of the estimates \texorpdfstring{$\vlcb_\ell, \vucb_\ell$}:}\label{app:vcb}
For completeness of the proof, we start by stating these essential concentration results, which were first proven in \cite{DBLP:journals/corr/AgrawalAGZ17a} (Corollary D.1):
\begin{lemma}[From \cite{DBLP:journals/corr/AgrawalAGZ17a}]\label{lemma:meancb}
Consider n i.i.d geometric random variables $X_1, \ldots, X_n$ with distribution $\prob(X_i = m) = (1 - p)^m p , \forall m = \{0, 1, \ldots \}$. Let $\mu = \mathbb{E}(X_i) = \frac{1 - p}{p}$ and $\tilde X  = \frac{\sum_{i = 1}^n X_i}{n}$. For $\mu \leq 1$, we have:
\begin{enumerate}
    \item $\prob\left( |\tilde X - \mu| > \sqrt{\tilde X \frac{48\log\left(\sqrt{N}T^2 + 1\right)}{n}} +\frac{48\log\left(\sqrt{N}T^2 + 1\right)}{n} \right) \leq \frac{6}{N T^2}$
    
    \item $\prob\left( |\tilde X - \mu| > \sqrt{\mu \frac{24 \log\left(\sqrt{N}T^2 + 1\right)}{n}} +\frac{48\log\left(\sqrt{N}T^2 + 1\right))}{n} \right) \leq \frac{4}{N T^2}$
    
    \item $\prob\left( \tilde X  \geq \frac{3 \mu}{2} + \frac{48\log\left(\sqrt{N}T^2 + 1\right)}{n} \right) \leq \frac{3}{N T^2}$
\end{enumerate}
\end{lemma}

\noindent From Lemma \ref{lemma:meancb}, we can derive the following concentration results on our estimators:

\begin{corollary}

\label{cor:vmeanbound}
For each $i$, epoch $\ell$, we have the following:
\begin{enumerate}
    \item $\prob\left(|\tilde v_{i, \ell} - v_i| >  \sqrt{\tilde v_{i, \ell}\frac{48\log\left(\sqrt{N}T^2 + 1\right)}{T_i(\ell)}} +\frac{48\log\left(\sqrt{N}T^2 + 1\right)}{T_i(\ell)} \right) \leq \frac{6}{N T^2}$
    
    \item $\prob\left( |\tilde v_{i, \ell} - v_i| >  \sqrt{v_i \frac{24\log\left(\sqrt{N}T^2 + 1\right)}{T_i(\ell)}} +\frac{48\log\left(\sqrt{N}T^2 + 1\right)}{T_i(\ell)} \right) \leq \frac{4}{N T^2}$
    
    \item $\prob\left( \tilde v_{i, \ell}  > \frac{3 v_i}{2} + \frac{48\log\left(\sqrt{N}T^2 + 1\right)}{T_i(\ell)} \right) \leq \frac{3}{N T^2}$
\end{enumerate}
\end{corollary}

\begin{proof}{Proof.}
The result follows from Lemma \ref{lemma:essential_properties}, and  Lemma \ref{lemma:meancb}. To address the randomness of $T_i(\ell)$, we use the law of total probability conditionally on  $T_i(\ell)$  apply the lemmas for each possible value of the sequence length in $[\ell]$ (since $T_i(\ell) \leq \ell$).
\hfill$\square$
\end{proof}
\noindent \textbf{Proof of Lemma \ref{lemma:vcb}.} We are now ready to prove Lemma \ref{lemma:vcb}.

\begin{proof}{Proof.}

The proof globally follows the proof of Lemma 4.1 in \cite{DBLP:journals/corr/AgrawalAGZ17a}. By Assumption \ref{assumption:vmax}, we already know that $v_i \in [0, 1]$ for all $i$. From Corollary \ref{cor:vmeanbound}, we obtain that for each $i$:

\begin{equation*}
    |v_i - \tilde v_{i, \ell}| \leq  \left( \tilde{v}_{i,\ell} \sqrt{\frac{48\log\left(\sqrt{N}T^2 + 1\right)}{T_i(\ell)}} + \frac{48\log\left(\sqrt{N}T^2 + 1\right)}{T_i(\ell)}\right)\; \; \; \text{w.p. } \geq1 - \tfrac{6}{N T^2}.
\end{equation*}

\noindent Hence, for each $i\in [N], \ell\in [L]$, $v_i \in [\vlcb_{i, \ell}, \vucb_{i, \ell}]$ with probability at least $1 - \tfrac{6}{N T^2}$.

\noindent Moreover,
\begin{align}
    \max\left\{|\vucb_{i, \ell} - v_i|, |\vlcb_{i, \ell} - v_i|\right\} &\leq \max\left\{|\vucb_{i, \ell} - \tilde v_{i, \ell}|, |\vlcb_{i, \ell} - \tilde v_{i, \ell}|\right\} + |\tilde v_{i, \ell} - v_i| \nonumber\\
    &= \sqrt{\tilde{v}_{i,\ell}\frac{48\log\left(\sqrt{N}T^2 + 1\right)}{T_i(\ell)}} + \frac{48\log\left(\sqrt{N}T^2 + 1\right)}{T_i(\ell)} + |\tilde v_{i, \ell} - v_i|\label{eq:max_ucb-vi}.
\end{align}

We can bound each term adequately. For the first two terms of (\ref{eq:max_ucb-vi}), we use the third inequality in Corollary \ref{cor:vmeanbound}, and combine it with the inequality $\sqrt{a} + \sqrt{b} \geq \sqrt{a + b}$:
\small{
\begin{align*}
    \prob\left( \sqrt{48 \tilde v_{i,\ell} \frac{\log\left(\sqrt{N}T^2 + 1\right)}{T_i(\ell)}} + \frac{48 \log\left(\sqrt{N}T^2 + 1\right)}{T_i(\ell)} >  \sqrt{\frac{72 v_i \log\left(\sqrt{N}T^2 + 1\right)}{T_i(\ell)}} + \frac{96 \log\left(\sqrt{N}T^2 + 1\right)}{T_i(\ell)}\right) \leq \frac{3}{N T^2}.
\end{align*}}
\normalsize
\noindent The last term of (\ref{eq:max_ucb-vi}) can be bounded by using the second inequality in \ref{cor:vmeanbound}. Thus, by doing a union bound over the $N L \leq N T$ terms $\{\mathcal{A_{i, \ell}}\}_{i \in [N], \ell \in [L]}$, and by choosing $C_1 = \sqrt{72} + \sqrt{24}$ and $C_2 = 144$, we get:
\begin{equation*}
    \prob(\mathcal{A}^{\sf CB1} \cap \mathcal{A}^{\sf CB2}) \geq 1 - \frac{3 + 4 + 6}{N T^2} N T= 1 - \frac{13}{N T}.
\end{equation*}
\noindent which completes the proof of Lemma  \ref{lemma:vcb}.
\hfill$\square$
\end{proof}
\section{Poof of structural lemmas}\label{app:lemmas}
\noindent\textbf{Proof of Lemma \ref{lem:OPT-LP}.}

% \begin{replemma}{lem:OPT-LP}
% The expected revenue of any non-anticipative policy $\mathcal{A}$ is less than $ \text{OPT}$.
% \end{replemma}

\begin{proof}{Proof.}
Consider any non-anticipative algorithm $\mathcal{A}$, and $\{S_t, c_t\}_{t = 1\ldots T}$ be the assortments offered by $\mathcal{A}$ and the customer's purchasing decisions. Note that these might be randomized and $S_t$ might depend on $S_1, c_1, \ldots, S_{t-1}, c_{t-1}$.

\noindent Set $y_S = \frac{1}{T}\sumt \prob(S_t = S)$,  for each $S \in \mathcal{S}$. Clearly, $y \geq 0$. We also have:
\begin{equation*}
    \sumS y_S = \frac{1}{T}\sumt \sumS \prob(S_t = S) = 1 .
\end{equation*}

\noindent Moreover, since the inventory constraints are respected by $\mathcal{A}$, the following inequalities hold for each product $i \in [N]$:
\begin{equation*}
    \sumt \mathbbm{1}_{c_t = i} \leq q_i T.
\end{equation*}

\noindent Thus, they hold in expectation:
\begin{align*}
    q_i T \geq \mathbb{E}_{\mathcal{A}}\left(\sumt \mathbbm{1}_{c_t}\right)= \sumt \mathbb{E}_{\mathcal{A}}\big(\mathbb{E}_{\mathcal{A}}(\mathbbm{1}_{c_t=i}|S_t)\big)= \sumt \mathbb{E}_{\mathcal{A}}\big(\pi(i, S_t)\big)= \sumt \sumS y_S\pi(i, S) = T \sumS y_S \pi(i, S).
\end{align*}

\noindent Hence $y$ is feasible in LP$(\pi)$. Thus, Its associated objective is less than $r_{\sf opt}$:

\begin{equation*}
    r_{\sf opt} \geq \sum_S y_S R(S).
\end{equation*}

\noindent However, $T \sumS y_S R(S)$ is exactly the revenue generated by $\mathcal{A}$, since:

\begin{equation*}
    \mathbb{E}_{\mathcal{A}}\left(\sumt r_{c_t}\right) = \sumt \mathbb{E}_{\mathcal{A}}\big( \mathbb{E}_{\mathcal{A}}(r_{c_t} |S_{t})\big)= \sumt \mathbb{E}_{\mathcal{A}}\big(R(S_t)\big)= T \sumS \left(\frac{1}{T}\sumt \prob(S_t = S)\right)R(S).
\end{equation*}

\noindent This concludes the proof.
\hfill$\square$
\end{proof}
\noindent \textbf{Proof of Lemma \ref{lem:lipshitz}.}
\begin{proof}{Proof.}
For an item $i$, an assortment $S \subset [N]$ with $i \in S$. For convenience, we drop the dependence in $\ell$, and we set:
\begin{equation*}
    \epsilon^+ = v^{\sf UCB} - v \qquad \text{and} \qquad
    \epsilon^- = v - v^{\sf LCB}.
\end{equation*}
We have:
\begin{align*}
    \pi^{\sf UCB}(i, S) - \pi(i, S)  &= \frac{v^{\sf UCB}_{i}}{1 + \sum_{j \in S} v^{\sf LCB}_{j}} - \frac{v_i}{1 + V(S)}\\
    &= \frac{v_i + \epsilon^+_{i}}{1 + V(S) -\epsilon^-(S)} - \frac{v_i}{1 + V(S)} \\
    &= v_i \left(\frac{1}{1 + V(S) -\epsilon^-(S)} - \frac{1}{1 + V(S)}\right) + \frac{\epsilon^+_{i}}{1 + V(S) -\epsilon^-(S)} \\
    &= \frac{v_i\epsilon^-(S)}{\left(1 + V(S)\right)\left(1 + V(S) -\epsilon^-(S)\right)} + \frac{\epsilon^+_{i}}{1 + V(S) -\epsilon^-(S)} \\
    &= \frac{\pi(i, S)\epsilon^-(S)}{1 + V(S) -\epsilon^-(S)} + \frac{\epsilon^+_{i}}{1 + V(S) -\epsilon^-(S)} \\
    &= \frac{\pi(i, S)\epsilon^-(S) + \epsilon^+_{i}}{1 + V(S) -\epsilon^-(S)}.
\end{align*}
Conditionally on $\mathcal{A}^{\sf CB1}$, we have $0 \leq \epsilon^-, \epsilon^+$, therefore:
\begin{equation*}
    \frac{\pi(i, S)\epsilon^-(S) + \epsilon^+_{i}}{1 + V(S) -\epsilon^-(S)} \leq \frac{\pi(i, S)\epsilon^-(S) + \epsilon^+_{i}}{1 + V(S)} = \frac{\left(v^{\sf UCB}_i - v_i\right) + \pi(i, S)\sum_{j \in S}\left(v_j - v^{\sf LCB}_j\right)}{1 + V(S)},
\end{equation*}
\noindent which concludes the proof of Lemma \ref{lem:lipshitz}.
\hfill$\square$
\end{proof}
\begin{lemma}[Improved upper bound]\label{lemma:epsilon}
    Conditionally on $\mathcal{A}^{\sf CB2}$, for each $i \in [N]$ and $\ell \geq 0$ with $i \in S_\ell$, we have
    \begin{equation*}
        \epsilon_{i, \ell} \leq \Tilde{C}_1 v_i \sqrt{\frac{\log{({\sqrt{N}}T^2+1)}}{N_i(\ell)}} + \Tilde{C}_2 v_i \left(\frac{\log{({\sqrt{N}}T^2+1)}}{N_i(\ell)}\right) + \Tilde{C}_3 \left(\frac{\log{({\sqrt{N}}T^2+1)}}{T_i(\ell)}\right)
    \end{equation*}
    where $\Tilde{C}_1, \Tilde{C}_2$ and $\Tilde{C}_3$ are universal constants. In particular, we set $\tilde{C}_1 = 8C_1, \Tilde{C}_2 = 8C_1^2$ and $\Tilde{C}_3 = C_2$.
\end{lemma}
\begin{proof}{Proof.}
    Conditionally on $\mathcal{A}^{\sf CB2}$, we have
    \begin{align*}
       \epsilon_{i, \ell} &\leq C_1 \sqrt{\frac{v_i\log{({\sqrt{N}}T^2+1)}}{T_i(\ell)}} + C_2\left(\frac{\log{({\sqrt{N}}T^2+1)}}{T_i(\ell)}\right) \\
       &= C_1 \sqrt{\log{({\sqrt{N}}T^2+1)}} \sqrt{\frac{v_i \tilde{v}_{i, \ell}}{N_i(\ell)}} + C_2\left(\frac{\log{({\sqrt{N}}T^2+1)}}{T_i(\ell)}\right) \\
       &\leq  C_1 \sqrt{\log{({\sqrt{N}}T^2+1)}}\sqrt{\frac{v_i (v_i + \epsilon_{i, \ell})}{N_i(\ell)}} + C_2\left(\frac{\log{({\sqrt{N}}T^2+1)}}{T_i(\ell)}\right) \\
       &\leq C_1 v_i \sqrt{\frac{\log{({\sqrt{N}}T^2+1)}}{N_i(\ell)}} + C_1 \sqrt{\frac{v_i\log{({\sqrt{N}}T^2+1)}}{N_i(\ell)}}\sqrt{\epsilon_{i, \ell}}+ C_2\left(\frac{\log{({\sqrt{N}}T^2+1)}}{T_i(\ell)}\right)
    \end{align*}
    where the first step stems from Lemma \ref{lemma:vcb}, the second stems from $N_i(\ell) = \Tilde{v}_{i, \ell} T_i(\ell)$, the third step stems from $\Tilde{v}_{i, \ell} \leq \vucb_{i, \ell} \leq v_i + \epsilon_{i, \ell}$, and the last step stems from $\sqrt{a + b} \leq \sqrt{a} + \sqrt{b}$ for $a, b \geq 0$. The last inequality implies that $\sqrt{\epsilon_{i, \ell}}$ is a solution to the following quadratic system:
    \begin{equation}\label{eq:quadratic}
        X^2 - A X - B \leq 0, \qquad X \geq 0,
    \end{equation}
    where
    \begin{align*}
        A = C_1 \sqrt{\frac{v_i\log{({\sqrt{N}}T^2+1)}}{N_i(\ell)}} \quad \text{ and } B = C_1 v_i \sqrt{\frac{\log{({\sqrt{N}}T^2+1)}}{N_i(\ell)}} + C_2\left(\frac{\log{({\sqrt{N}}T^2+1)}}{T_i(\ell)}\right).
    \end{align*}
    The quadratic system \eqref{eq:quadratic} implies $0 \leq X \leq \sqrt{A^2 + 4B} + A \leq A + 2 \sqrt{B} + A = 2(A + \sqrt{B})$. Therefore
    \begin{align*}
        \epsilon_{i, \ell} &\leq \left(2(A + \sqrt{B})\right)^2 \\ &\leq 8(A^2 + B)\\ &=  \frac{8C_1^2 v_i \log{({\sqrt{N}}T^2+1)}}{N_i(\ell)} + 8 C_1 v_i \sqrt{\frac{\log{({\sqrt{N}}T^2+1)}}{N_i(\ell)}}+ C_2\left(\frac{\log{({\sqrt{N}}T^2+1)}}{T_i(\ell)}\right) \\
        &= \Tilde{C}_1 v_i \sqrt{\frac{\log{({\sqrt{N}}T^2+1)}}{N_i(\ell)}} + \Tilde{C}_2 v_i \left(\frac{\log{({\sqrt{N}}T^2+1)}}{N_i(\ell)}\right) + \Tilde{C}_3 \left(\frac{\log{({\sqrt{N}}T^2+1)}}{T_i(\ell)}\right).
    \end{align*}
    where the inequality stems from $(a+b)^2 \leq 2(a^2 + b^2)$ for $a, b \geq 0$ and where we set $\tilde{C}_1 = 8C_1, \Tilde{C}_2 = 8C_1^2$ and $\Tilde{C}_3 = C_2$. This completes the proof of \ref{lemma:epsilon}.
    \hfill $\square$
\end{proof}

\noindent \textbf{Proof of Lemma \ref{lemma:algebraic}.}
\begin{proof}{Proof}
    By setting
    \begin{equation*}
        \Delta =  \sum_{\ell \geq 1}|\mathcal{E}_\ell| \sum_{S \in \mathcal{S}}y_\ell(S)\pi^{\sf UCB}_{\ell-1}(i, S) +  A_{q,2} \sqrt{\sum_{\ell \geq 1} |\mathcal{E}_\ell| \sum_{S \in \mathcal{S}}y_\ell(S)\pi^{\sf UCB}_{\ell-1}(i, S)} + A_{q, 0},
    \end{equation*}
    from Inequality \eqref{lemma:algebraic}, we have $\sqrt{\mathcal{I}_{i, T}}$ the solution to the system
    \begin{equation*}
        X \geq 0, \qquad X^2 - A_{q, 1} X - \Delta \leq 0.
    \end{equation*}
    This system implies that
    \begin{equation*}
        X \leq \frac{1}{2}\left(\sqrt{A_{q, 1}^2 + 4 \Delta} + A_{q, 1}\right),
    \end{equation*}
    which in turn implies, by using $\sqrt{a + b} \leq \sqrt{a} + \sqrt{b}$ for $a, b \geq 0$:
    \begin{align*}
        X^2 &\leq \frac{1}{4}\left(\sqrt{A_{q, 1}^2 + 4 \Delta} + A_{q, 1}\right)^2 \\ &= \Delta + \frac{A^2_{q, 1} + A_{q, 1}\sqrt{A_{q, 1}^2 + 4 \Delta}}{2}  \\ &\leq \Delta + \frac{A^2_{q, 1} + A_{q, 1}\sqrt{A_{q, 1}^2} + A_{q, 1}\sqrt{4 \Delta}}{2} \\ X^2 &\leq \Delta  + A_{q,1}\sqrt{\Delta} + A^2_{q, 1}.
    \end{align*}
    \noindent Hence:
    \begin{align*}
        \mathcal{I}_{i, T} &\leq \Delta  + A_{q,1}\sqrt{\Delta} + A^2_{q, 1} \\
        &= \sum_{\ell \geq 1}|\mathcal{E}_\ell| \sum_{S \in \mathcal{S}}y_\ell(S)\pi^{\sf UCB}_{\ell-1}(i, S) +  A_{q,2} \sqrt{\sum_{\ell \geq 1} |\mathcal{E}_\ell| \sum_{S \in \mathcal{S}}y_\ell(S)\pi^{\sf UCB}_{\ell-1}(i, S)} + A_{q, 0} \\
        &+ A_{q, 1}\sqrt{\sum_{\ell \geq 1}|\mathcal{E}_\ell| \sum_{S \in \mathcal{S}}y_\ell(S)\pi^{\sf UCB}_{\ell-1}(i, S) +  A_{q,2} \sqrt{\sum_{\ell \geq 1} |\mathcal{E}_\ell| \sum_{S \in \mathcal{S}}y_\ell(S)\pi^{\sf UCB}_{\ell-1}(i, S)} + A_{q, 0}} \\
        &+ A_{q, 1}^2 \\
        &\leq \sum_{\ell \geq 1}|\mathcal{E}_\ell| \sum_{S \in \mathcal{S}}y_\ell(S)\pi^{\sf UCB}_{\ell-1}(i, S) \\
        &+ \left(A_{q,2} + A_{q,1}\right)\sqrt{\sum_{\ell \geq 1}|\mathcal{E}_\ell| \sum_{S \in \mathcal{S}}y_\ell(S)\pi^{\sf UCB}_{\ell-1}(i, S)} \\
        &+ \left(A_{q,2} + A_{q,1}\sqrt{A_{q,2}}\right)\sqrt[4]{\sum_{\ell \geq 1}|\mathcal{E}_\ell| \sum_{S \in \mathcal{S}}y_\ell(S)\pi^{\sf UCB}_{\ell-1}(i, S)} \\
        &+ A_{q, 0} + A_{q, 1}\sqrt{A_{q, 0}} + A_{q,1}^2 \\
       &\leq q_i T - \omega_i q_i T  + (A_{q,1}+A_{q,2})\sqrt{q_i T} + \left(A_{q,2} + A_{q,1}\sqrt{A_{q,2}}\right) \sqrt[4]{q_i T} + (A_{q, 0} + A_{q, 1}\sqrt{A_{q, 0}} + A_{q,1}^2) \\
        &\leq q_i T - q_i T\left(\omega_i - \frac{A_1}{\sqrt{q_i T}} - \frac{A_{0}}{q_i T}\right) \\
         \mathcal{I}_{i, T}&\leq q_i T,
    \end{align*}
     where we set 
     \begin{align*}
         A_1 &= 2(A_{q, 1} + A_{q, 2}), \\
         A_0 &= A_{q, 0} + A_{q, 1}\sqrt{A_{q, 0}} + A_{q,1}^2.
     \end{align*}
    The first inequality stems from the previous system, the second inequality stems from $\sqrt{a + b} \leq \sqrt{a} + \sqrt{b}$, the third inequality stems from the feasibility of $y_\ell$ in LP$(\pi^{\sf UCB}_{\ell - 1})$, the fourth inequality stems from $q_i T \geq 1$, and the last inequality stems from the expression of $\omega_i$. This completes the proof of Lemma \ref{lemma:algebraic}.
     \hfill $\square$
\end{proof}

\noindent \textbf{Proof of Lemma \ref{lemma:errors}.}

\begin{proof}
\noindent \textbf{Bounding $\|\boldsymbol{\epsilon}_i\|_2$.}We recall Lemma \ref{lemma:vcb}, which states that conditionally on $\mathcal{A}^{\sf CB2}$, for all $i \in [N]$, $\ell \geq 0$,
\begin{equation*}
     \epsilon_{i, \ell} \leq C_1 \sqrt{\frac{v_i\log{({\sqrt{N}}T^2+1)}}{T_i(\ell)}} + C_2\left(\frac{\log{({\sqrt{N}}T^2+1)}}{T_i(\ell)}\right),
\end{equation*}
which implies that
\begin{equation}\label{eq:epsilonsq}
    \epsilon_{i, \ell}^2 \leq \frac{2C_1^2v_i\log{({\sqrt{N}}T^2+1)}}{T_i(\ell)} + 2C_2^2\left(\frac{\log{({\sqrt{N}}T^2+1)}}{T_i(\ell)}\right)^2,
\end{equation}
where we used $(a+b)^2 \leq 2(a^2 + b^2)$ for $a, b \geq 0$. Summing the previous inequality over $\ell \geq 0$ and using $v_i \leq 1$ implies
\begin{align*}
   \sum_{\ell \geq 0}\epsilon_{i, \ell}^2 &\leq 2C_1^2 \log\left(\sqrt{N}T^2 + 1\right)v_i \sum_{t = 1}^{T_i(L)}\frac{1}{t}+ 2C_2^2\log^2\left(\sqrt{N}T^2+ 1\right)\sum_{t = 1}^{T_i(L)}\frac{1}{t^2} \\ &\leq 2C_1^2 \log\left(\sqrt{N}T^2 + 1\right)v_i \times 2 \log(T_i(L))+ 2C_2^2\log^2\left(\sqrt{N}T^2+ 1\right)\times \frac{\pi^2}{6} \\
   &\leq 4(C_1^2 + C_2^2)\log^2\left(\sqrt{N}T^2 + 1\right),
\end{align*}
hence $ \|\boldsymbol{\epsilon}_i\|_2 \leq \sqrt{\sum_{\ell \geq 0} \epsilon_{i, \ell}^2} \leq 2\sqrt{C_1^2 + C_2^2}\log\left(\sqrt{N}T^2 + 1\right)$.

\noindent \textbf{Bounding $\|\boldsymbol{w}_i\|_2$.} To derive a bound on the $2-$norm of $\boldsymbol{w}_i$, we first need to bound $w_{i, \ell}^2$. To do so, we derive an intermediate bound on $w_{i, \ell}$. On the one hand, we have by applying Cauchy Shwartz inequality:
 \begin{align*}
   \sum_{j \in S_\ell} \epsilon_{j, \ell} &\leq C_1 \sqrt{\log{({\sqrt{N}}T^2+1)}}\sum_{j \in S_\ell} \sqrt{\frac{v_i}{T_i(\ell)}} + C_2\log{({\sqrt{N}}T^2+1)}\sum_{j \in S_\ell} \left(\frac{1}{T_i(\ell)}\right) \\
    &\leq C_1 \sqrt{\log(\sqrt{N}T^2 + 1)} \sqrt{\sum_{j \in S_\ell} v_j} \sqrt{\sum_{j \in S_\ell} \frac{1}{T_j(\ell)}} + C_2 \log(\sqrt{N}T^2 + 1) \sum_{j \in S_\ell} \frac{1}{T_j(\ell)}.
 \end{align*}
 On the other hand, notice that
 \begin{equation*}
     \pi(i, S_\ell)\sqrt{\sum_{j \in S_\ell} v_j} \leq \pi(i, S_\ell)\sqrt{1 + \sum_{j \in S_\ell} v_j} = \pi(i, S_\ell)\sqrt{\frac{v_i}{\pi(i, S_\ell)}} = \sqrt{v_i \pi(i, S_\ell)} \leq v_i.
 \end{equation*}
Therefore, by multiplying $\pi(i, S_\ell)$ in both sides in the previous inequality we obtain
 \begin{equation*}
     \pi(i, S_\ell) \sum_{j \in S_\ell} \epsilon_{j, \ell} \leq C_1 \sqrt{\log(\sqrt{N}T^2 + 1)} v_i \sqrt{\sum_{j \in S_\ell} \frac{1}{T_j(\ell)}} + C_2 \log(\sqrt{N}T^2 + 1) \sum_{j \in S_\ell} \frac{\pi(i, S_\ell)}{T_j(\ell)},
 \end{equation*}
 which yields the following intermediate bound on $w_{i, \ell} = \epsilon_{i, \ell} + \pi(i, S_\ell)\sum_{j \in S_\ell} \epsilon_{j, \ell}$ for each with $i \in S_\ell$:
 \small
 \begin{equation}\label{eq:wintermediate}
     w_{i, \ell} \leq C_1\sqrt{\log\left(\sqrt{N}T^2 + 1\right)}\left(\sqrt{\frac{v_i}{T_i(\ell)}} + v_i\sqrt{\sum_{j \in S_\ell} \frac{1}{T_j(\ell)}}\right) + C_2\log\left(\sqrt{N}T^2 + 1\right) \left(\frac{1}{T_i(\ell)} + \sum_{j \in S_\ell} \frac{1}{T_j(\ell)}\right).
 \end{equation}
 \normalsize
 By applying $(a+b)^2 \leq 2(a^2 + b^2)$ twice on Equation \eqref{eq:wintermediate}, we obtain for every $\ell \geq 0$ and $i \in S_\ell$:
 \begin{align*}
     w_{i, \ell}^2 &\leq 4C_1^2 \log\left(\sqrt{N}T^2 + 1\right)\left(\frac{v_i}{T_i(\ell)} + v_i^2 \sum_{j \in S_\ell} \frac{1}{T_j(\ell)}\right) + 4C_2^2 \log^2\left(\sqrt{N}T^2 + 1\right) \left(\frac{1}{T^2_i(\ell)} + \left(\sum_{j \in S_\ell} \frac{1}{T_j(\ell)}\right)^2\right).
 \end{align*}
 By summing over $\ell \geq 0$ we obtain,
 \begin{align*}
     &\sum_{\ell \geq 0} \frac{\mathbbm{1}_{i \in S_\ell}}{T_i(\ell)} \leq \sum_{t = 1}^{T_i} \frac{1}{t} \leq 2\log T_i \leq 2\log T \leq \log\left(\sqrt{N}T^2 + 1\right), \\
     &\sum_{\ell \geq 0}\sum_{j \in S_\ell}\frac{1}{T_j(\ell)} = \sum_{j \in [N]}\sum_{\ell \geq 0} \frac{\mathbbm{1}_{j \in S_\ell}}{T_j(\ell)} \leq 2 N \log T \leq N \log\left(\sqrt{N}T^2 + 1\right), \\
     &\sum_{\ell \geq 0}\frac{\mathbbm{1}_{i\in S_\ell}}{T^2_i(\ell)}  \leq \sum_{t \geq 0} \frac{1}{t^2} =\frac{\pi^2}{6} \leq 10/6 \leq 2, \\
     &\sum_{\ell \geq 0}\left(\sum_{j \in S_\ell} \frac{1}{T_j(\ell)}\right)^2 \leq \sum_{\ell \geq 0} \left(\sum_{j \in S_\ell} 1 \right)\left(\sum_{j \in S_\ell}\frac{1}{T^2_j(\ell)}\right) \leq \left(K\right) \sum_{\ell \geq 0} \sum_{j \in S_\ell} \frac{1}{T^2_j(\ell)} \leq 2 N K.
 \end{align*}
 Therefore, by recalling that $\boldsymbol{v} \leq 1$, we obtain
 \begin{equation*}
     \|\boldsymbol{w}_i\|^2_2 = \sum_{\ell \geq 0} w_{i, \ell}^2 \leq 4\log^2\left(\sqrt{N}T^2 + 1\right)\left(C_1^2(N+1) + 2C_2^2(1 + N K)\right),
 \end{equation*}
 which implies by applying $\sqrt{a + b} \leq \sqrt{a} + \sqrt{b}$ for $a, b \geq 0$:
 \begin{align*}
       \|\boldsymbol{w}_i\|_2 &\leq 2\log\left(\sqrt{N}T^2 + 1\right)\sqrt{\left(C_1^2(N+1) + 2C_2^2(1 + N K)\right)} \\
        &\leq 2\log\left(\sqrt{N}T^2 + 1\right) \left(\sqrt{C_1^2(N+1)} + \sqrt{2C_2^2(1 + N K)}\right) \\&\leq 2\sqrt{2}(C_1+ C_2\sqrt{2})\log\left(T^4 + 1\right)\sqrt{NK)},
 \end{align*}
 which concludes the proof of bounding $\|\boldsymbol{w}_i\|_2$. In order to bound the $\|\boldsymbol{\epsilon}_i\|_1$ and $\|\boldsymbol{w}_i\|_1$, we will need to improve the bound on $\epsilon_{i, \ell}$ given in Lemma \ref{lemma:vcb}. In particular, we need a better scaling in $v_i$. This is formalized in Lemma \ref{lemma:epsilon}.
 
\noindent \textbf{Bounding $\|\boldsymbol{\epsilon}_i\|_1$.} From Lemma \ref{lemma:epsilon}, we have
\begin{align*}
    \sum_{\ell \geq 0} \epsilon_{i, \ell}
    &\leq \Tilde{C}_1  v_i\sqrt{\log\left(\sqrt{N}T^2 + 1\right)}\sum_{\ell \geq 0}\frac{\mathbbm{1}_{i \in S_\ell}}{\sqrt{N_i(\ell)}} &(a)\\
    &+ \Tilde{C}_2 v_i \log\left(\sqrt{N}T^2 + 1\right) \sum_{\ell \geq 0}\frac{\mathbbm{1}_{i \in S_\ell}}{N_i(\ell)}&(b)\\
    &+  \Tilde{C}_3 \log\left(\sqrt{N}T^2 + 1\right) \sum_{\ell \geq 0}\frac{\mathbbm{1}_{i \in S_\ell}}{T_i(\ell)},&(c)
\end{align*}
We bound each of the three terms separately. First, we have
\begin{equation*}
    (a): \sum_{\ell \geq 0}\frac{\mathbbm{1}_{i \in S_\ell}}{\sqrt{N_i(\ell)}} \leq \sum_{t = 1}^{N_i(L)}\frac{1}{\sqrt{t}} \leq 2 \sqrt{N_i(L)} \leq 2\sqrt{T},
\end{equation*}
Second,
\begin{equation*}
    (b) : \sum_{\ell \geq 0}\frac{\mathbbm{1}_{i \in S_\ell}}{N_i(\ell)} \leq  \sum_{t = 1}^{N_i(L)}\frac{1}{t} \leq 2\log N_i(L) \leq 2\log T. \text{ Similarly, } (c): \sum_{\ell \geq 9}\frac{\mathbbm{1}_{i \in S_\ell}}{T_i(\ell)} \leq  \sum_{t = 1}^{T_i(L)}\frac{1}{t} \leq 2\log T_i(L) \leq 2\log T.
\end{equation*}
Therefore,
\begin{align*}
     \|\boldsymbol{\epsilon}_i\|_1 &= \sum_{\ell \geq 0} \epsilon_{i, \ell}\\ &\leq 2\Tilde{C}_1 v_i\sqrt{\log\left(\sqrt{N}T^2 + 1\right)T} + \Tilde{C}_2 v_i \log\left(\sqrt{N}T^2 + 1\right)\log T + \Tilde{C}_3 \log\left(\sqrt{N}T^2 + 1\right) \log T \\
     &\leq 2\Tilde{C}_1 v_i\sqrt{\log\left(\sqrt{N}T^2 + 1\right)T} + (\Tilde{C}_2v_i + \Tilde{C}_3)\log^2\left(\sqrt{N}T^2 + 1\right),
\end{align*}
which provides the bound for $ \|\boldsymbol{\epsilon}_i\|_1$. 

\noindent \textbf{Bounding $\|\sqrt{\boldsymbol{\epsilon}_i}\|_1$.} 

Similarly to how we bounded $ \|\boldsymbol{\epsilon}_i\|_1$, we have:
\begin{align*}
    \sum_{\ell \geq 0} \sqrt{\epsilon_{i, \ell}}
    &\leq \sqrt{\Tilde{C}_1  v_i\sqrt{\log\left(\sqrt{N}T^2 + 1\right)}}\sqrt{\sum_{\ell \geq 0}\frac{\mathbbm{1}_{i \in S_\ell}}{\sqrt{N_i(\ell)}}}\\
    &+ \sqrt{\Tilde{C}_2 v_i \log\left(\sqrt{N}T^2 + 1\right)} \sqrt{\sum_{\ell \geq 0}\frac{\mathbbm{1}_{i \in S_\ell}}{N_i(\ell)}}\\
    &+  \sqrt{\Tilde{C}_3 \log\left(\sqrt{N}T^2 + 1\right)}\sqrt{\sum_{\ell \geq 0}\frac{\mathbbm{1}_{i \in S_\ell}}{T_i(\ell)}} \\
    &\leq \sqrt{2\Tilde{C}_1  \sqrt{\log\left(\sqrt{N}T^2 + 1\right)}}\sqrt{v_i} \sqrt[4]{T} + \left(\sqrt{2\Tilde{C}_2} + \sqrt{2\Tilde{C}_3}\right) \sqrt{\log\left(\sqrt{N}T^2 + 1\right)\log T}
\end{align*}
where we used $\sqrt{a + b + c} \leq \sqrt{a} + \sqrt{b} + \sqrt{c}$ along with the inequalities used to derive the upper bound on $\|\boldsymbol{\epsilon}_i\|_1$.

It remains to bound $\|\boldsymbol{w}\|_1$. This term eventually represents the highest rate in the regret. Consequently, bounding it requires extra care.

\noindent \textbf{Bounding $\|\boldsymbol{w}_i\|_1$.} We have by definition of $\boldsymbol{w}$:
\begin{equation*}
    \|\boldsymbol{w}_i\|_1 = \sum_{\ell \geq 0} \epsilon_{i, \ell} + \sum_{\ell \geq 0}\pi(i, S_\ell)\sum_{j \in S_\ell}\epsilon_{j, \ell}
\end{equation*}
We've already established an upper bound on the first sum:
\begin{equation*}
    \sum_{\ell \geq 0} \epsilon_{i, \ell}  = \|\boldsymbol{\epsilon}_i\|_1\leq 2\Tilde{C}_1 v_i\sqrt{\log\left(\sqrt{N}T^2 + 1\right)T} + (\Tilde{C}_2v_i + \Tilde{C}_3)\log^2\left(\sqrt{N}T^2 + 1\right).
\end{equation*}
We focus on the second sum $\sum_{\ell \geq 0}\pi(i, S_\ell)\sum_{j \in S_\ell}\epsilon_{j, \ell}$. From Lemma \ref{lemma:epsilon}, we have
\begin{align*}
    \sum_{\ell \geq 0}\pi(i, S_\ell)\sum_{j \in S_\ell}\epsilon_{j, \ell} &\leq \Tilde{C}_1 v_i\sqrt{\log\left(\sqrt{N}T^2 + 1\right)}\sum_{\ell \geq 0}\sum_{j \in S_\ell} \pi(j, S_\ell) \sqrt{\frac{\mathbbm{1}_{i\in S_\ell}}{N_j(\ell)}} &(a)\\
    &+\Tilde{C}_2 v_i\log\left(\sqrt{N}T^2 + 1\right)\sum_{\ell \geq 0}\sum_{j \in S_\ell}  \pi(j, S_\ell)\left(\frac{\mathbbm{1}_{i\in S_\ell}}{N_j(\ell)}\right)  &(b)\\&+ \Tilde{C}_3\log\left(\sqrt{N}T^2 + 1\right)\sum_{\ell \geq 0}\sum_{j \in S_\ell}\frac{\mathbbm{1}_{i \in S_\ell}}{T_j(\ell)}, &(c)
\end{align*}
where we used $v_i \pi(j, S) = v_j \pi(i, S)$ for $i, j \in S$.
We once again bound each of the three sums separately. First, we have
\begin{align*}
    (a): \sum_{\ell \geq 0}\sum_{j \in S_\ell} \pi(j, S_\ell) \sqrt{\frac{\mathbbm{1}_{i\in S_\ell}}{N_j(\ell)}} &= \sum_{j \in [N]}v_j\sum_{\ell \geq 0} \pi(0, S_\ell) \sqrt{\frac{\mathbbm{1}_{i\in S_\ell}}{N_j(\ell)}} \\ &\leq \sum_{j \in [N]} 2 v_j \sqrt{N_j(L)}\\ &\leq 2 \sqrt{\sum_{j \in [N]}v_j^2} \sqrt{\sum_{j \in [N]}N_j(L)}\\ &\leq 2\|\boldsymbol{v}\|_2\sqrt{T}.
\end{align*}
where the first step stems from $\pi(j, S_\ell) = v_j \pi(0, S_\ell) \leq \frac{v_j}{1 + v_j} \leq v_j$, where the second step stems from $\sum_{x = 1}^X \frac{1}{\sqrt{x}} \leq 2 \sqrt{X}$, where the third step stems from Cauchy Shwartz, and where the last step stems from $N_0(L) + \sum_{j \in [N]} N_j(L) \leq T$ (we offer at most one product per time step). In a similar fashion, we have
\begin{align*}
    (b): \sum_{\ell \geq 0}\sum_{j \in S_\ell}  \pi(j, S_\ell)\left(\frac{\mathbbm{1}_{i\in S_\ell}}{N_j(\ell)}\right) &= \sum_{j \in [N]} v_j \sum_{\ell \geq 0}\frac{\pi(0, S_\ell)\mathbbm{1}_{i \in S_\ell}}{N_i(\ell)} \\
    &\leq \sum_{j \in [N]} v_j (\log(N_j(L)) + 1)\\ &\leq \|\boldsymbol{v}\|_1 \log (T+1) \leq 2  \|\boldsymbol{v}\|_1 \log T.
\end{align*}
The third sum is upper bounded as follows
\begin{align*}
    (c): \sum_{\ell \geq 0}\sum_{j \in S_\ell} \frac{\mathbbm{1}_{i \in S_\ell}}{T_j(\ell)} = \sum_{j \in [N]} \sum_{\ell \geq 0}\frac{\mathbbm{1}_{i,j \in S_\ell}}{T_j(\ell)} \leq \sum_{j \in [N]} 2\log T_j(L) \leq 2 N \log T.
\end{align*}
Combining all previous three inequalities yields the following upper bound on $\|\boldsymbol{w}_i\|_1$:
\begin{align*}
    \|\boldsymbol{w}_i\|_1 &\leq 2\Tilde{C}_1\sqrt{\log\left(\sqrt{N}T^2 + 1\right)} v_i \|\boldsymbol{v}\|_2\sqrt{T} + \log^2\left(\sqrt{N}T^2 + 1\right)\left(2\Tilde{C}_2 v_i\|\boldsymbol{v}\|_1 + 2\Tilde{C}_3 N\right).
\end{align*}
which completes the proof of Lemma \ref{lemma:errors}.
\hfill $\square$
\end{proof}

\noindent \textbf{Proof of Lemma \ref{lemma:firstterm}}
\begin{proof}{Proof.}
We construct an alternative distribution $\Tilde{y}_{\ell}$, defined as a solution to the following system:
\begin{equation}\label{eq:ytilde}
\forall i \in [N], \quad \sum_{S \in \mathcal{S}} \Tilde{y}_{\ell}(S) \pi^{\sf OPT}_{\ell - 1}(i, S) = (1 - \omega_{i})\sum_{S \in \mathcal{S}} y^{\sf OPT}(S)\pi(i, S).
\end{equation}
The existence of $\Tilde{y}$ is due to the $\pi^{\sf UCB} \geq \pi$, the small size of the support of $y^{\sf OPT}$, and $1 - \omega_i < 1$. Notice that from the system \eqref{eq:ytilde} and the feasibility of $y^{\sf OPT}$ in LP$(\pi)$, the new distribution $\Tilde{y}$ is feasible in LP$(\pi^{\sf UCB}_{\ell - 1})$. This way, conditionally on $\mathcal{A}$, and for a fixed $\ell \geq 1$:
\small
\begin{equation*}
     \sum_{i \in [N]}r_i \sum_{S \in \mathcal{S}} y_\ell(S) \pi^{\sf UCB}_{\ell-1}(i,S) \geq  \sum_{i \in [N]} r_i \sum_{S \in \mathcal{S}} \tilde{y}_\ell(S)\pi^{\sf UCB}_{\ell-1}(i, S) = \sum_{i \in [N]}(1 - \omega_i) r_i \sum_{S \in \mathcal{S}} y^{\sf OPT}(S)\pi(i, S) = r_{\sf opt} - r_{\sf disc}.
\end{equation*}
\normalsize
The first step stems from the feasibility of $\Tilde{y}_\ell$ in \text{LP$(\pi^{\sf UCB}_{\ell-1})$} combined with the optimality of $y_\ell$, and the second step stems from the definition of $\Tilde{y}_\ell$. This completes the proof of Lemma \ref{lemma:firstterm}.
\hfill $\square$
\end{proof}

\subsection{Details of calculations.}\label{app:calculations}

\noindent \textbf{Total magnitude of the noise.}
We first apply Lemma \ref{lem:deltarandommnl} and Corollary \ref{eq:shifterror}. Conditionally on $\mathcal{A}$, we have
\small{
\begin{align*}
    \left|\delta_{i, T}^{\sf random}\right| + \left|\delta_{i, T}^{\sf mnl}\right| + \left|\delta_{i, T}^{\sf shift}\right| &\leq 9\sqrt{C_0 v_i T \log T} + 9C_0 \log T +3\sqrt{C_0  \log T \left(\sum_{\ell \geq 1}|\mathcal{E}_{\ell}| \epsilon_{i, \ell-1}\right)} + \sum_{\ell \geq 1}\frac{|\mathcal{E}_{\ell}|}{1 + V(S_\ell)}w_{i, \ell-1} \\
    &\leq 9\sqrt{C_0  v_i T \log T} + 9C_0 \log T \\ &+3\sqrt{C_0  \log T\left(K\right)(\left(2\log(N T)\|\boldsymbol{\epsilon}_i\|_2 + \|\boldsymbol{\epsilon}_i\|_1)\right)} \\ &+ 2\log T\|\boldsymbol{w}_i\|_2 + \|\boldsymbol{w}_i\|_1 \\
    &\leq  9\sqrt{C_0  v_i T \log T} + 9C_0 \log T &(a)\\ &+ 3\sqrt{C_0  \log T\left(K\right)}\left(\sqrt{2\log(N T)\|\boldsymbol{\epsilon}_i\|_2} + \sqrt{\|\boldsymbol{\epsilon}_i\|_1)}\right) &(b)\\ &+ 2\log T\|\boldsymbol{w}_i\|_2 + \|\boldsymbol{w}_i\|_1. &(c)
\end{align*}
}
\normalsize
We separated the different terms into three lines ($(a), (b)$ and $(c)$), and the next steps will consist upper bounding each line separately. To simplify notations, we will introduce near-constants that will encompass numerical constants and log factors in $N$ and $T$. We set:
\begin{equation*}
    A_{0,1} := 9C_0  \log T, \qquad A_{1,1} := 9\sqrt{C_0  \log T}.
\end{equation*}
This way 
\begin{equation*}
    (a) =  9\sqrt{C_0  v_i T \log T} + 9C_0 \log T = A_{0, 1} + A_{1, 1}\sqrt{v_i T}.
\end{equation*}
Bounding $(b)$ and $(c)$ requires Lemma \ref{lemma:errors}. First, by using $\sqrt{a+b} \leq \sqrt{a} + \sqrt{b}$ for $a, b \geq 0$, we have:
\begin{align*}
    \sqrt{\|\boldsymbol{\epsilon}_i\|_2} &\leq \sqrt{\sqrt{2C_1^2 + 4C_2^2}\log\left(\sqrt{N}T^2 + 1\right)}\\ &= \sqrt[4]{2C_1^2 + 4C_2^2}\sqrt{\log\left(\sqrt{N}T^2 + 1\right)}, \\
    \sqrt{\|\boldsymbol{\epsilon}_i\|_1} &\leq \sqrt{2\Tilde{C}_1 v_i\sqrt{\log\left(\sqrt{N}T^2 + 1\right)N_i} + (\Tilde{C}_2v_i + \Tilde{C}_3)\log^2\left(\sqrt{N}T^2 + 1\right)} \\
    &\leq\sqrt{2\Tilde{C}_1 v_i\sqrt{\log\left(\sqrt{N}T^2 + 1\right)N_i}} + \sqrt{(\Tilde{C}_2v_i + \Tilde{C}_3)\log^2\left(\sqrt{N}T^2 + 1\right)} \\
    &\leq \sqrt{2\Tilde{C}_1} \sqrt{v_i}\sqrt[4]{\log\left(\sqrt{N}T^2 + 1\right)T} + \sqrt{(\Tilde{C}_2 + \Tilde{C}_3)\log^2\left(\sqrt{N}T^2 + 1\right)}.
\end{align*}
where the last inequality stems from $\boldsymbol{v} \leq 1$ and $N_i \leq T$. Therefore, by choosing
\begin{align*}
    A_{0, 2} &= 3\sqrt{C_0  \log T\left(\log(NT)\right)}\sqrt[4]{2C_1^2 + 4C_2^2}\sqrt{\log\left(\sqrt{N}T^2 + 1\right)}, \\
    A_{0, 3} &= 3\sqrt{C_0  \log T\left(2\log(NT)\right)}\sqrt{\Tilde{C}_2 + \Tilde{C}_3}\log\left(\sqrt{N}T^2 + 1\right), \\
    A_{1/2} &=  3\sqrt{C_0  \log T\left(2\log(NT)\right)} \sqrt{2\Tilde{C}_1}\sqrt[4]{\log\left(\sqrt{N}T^2 + 1\right)}.
\end{align*}
\normalsize
we obtain 
\begin{align*}
  (b) &=  3\sqrt{C_0  \log T\left(K\right)}\left(\sqrt{2\log(N T)\|\boldsymbol{\epsilon}_i\|_2} + \sqrt{\|\boldsymbol{\epsilon}_i\|_1)}\right) \\ &\leq A_{0,2}+A_{0,3}\sqrt{K} + A_{1/2}\sqrt{v_i}\sqrt{K}\sqrt[4]{T} \\
   &\leq A_{0,2}+A_{0,3}\sqrt{N} + A_{1/2}\sqrt{v_i}\sqrt{K}\sqrt[4]{T}.
\end{align*}
We now bound $(c) = 2\log T \|\boldsymbol{w}_i\|_2 + \|\boldsymbol{w}_i\|_1$. By choosing:
\begin{align*}
    A_{0, 4} &:= \left(2\Tilde{C}_2 + 2\Tilde{C}_3 \right)\log^2\left(\sqrt{N}T^2 + 1\right), \\
    A_{0, 5} &:= 4\sqrt{2}(C_1+ C_2\sqrt{2})\log T\log\left(N T^2 + 1\right),\\
    A_{1, 2} &:= 2\Tilde{C}_1 \sqrt{\log\left(\sqrt{N}T^2 + 1\right)}.
\end{align*}
we obtain from Lemma \ref{lemma:errors} that conditionally on $\mathcal{A}$:
\begin{equation*}
    (c) = 2\log T \|\boldsymbol{w}_i\|_2 + \|\boldsymbol{w}_i\|_1 \leq \left(A_{0, 4} + A_{0, 5}\right)N + A_{1,2}v_i \|\boldsymbol{v}\|_2 \sqrt{T}.
\end{equation*}
Combining all previous inequalities implies that conditionally on $\mathcal{A}$,
\small
\begin{align*}
    \left|\delta_{i, T}^{\sf random}\right| + \left|\delta_{i, T}^{\sf mnl}\right| + \left|\delta_{i, T}^{\sf shift}\right| &\leq (a) + (b) + (c) \\
    &\leq \left(A_{0, 1} + A_{0, 2}\right) +  (A_{0, 3} + A_{0, 4} + A_{0, 5})N + A_{1/2}\sqrt{v_i}\sqrt{K}\sqrt[4]{T} + (A_{1, 1} +A_{1,2})v_i \|\boldsymbol{v}\|_2\sqrt{T}.
\end{align*}
\normalsize
Therefore, by setting $A^{\sf t}_{0} = A_{0, 1} + A_{0,2} + A_{0,3} + A_{0,4} + A_{0,5}$ and $A^{\sf t}_1 := A_{1, 1} +A_{1,2}+A_{1,3}$, we obtain conditionally on $\mathcal{A}$,
\begin{equation*}
    \left|\delta_{i, T}^{\sf random}\right| + \left|\delta_{i, T}^{\sf mnl}\right| + \left|\delta_{i, T}^{\sf shift}\right| \leq A^{\sf t}_0 N+ A_{1/2} \sqrt{v_i}\sqrt{K}\sqrt[4]{T} + A^{\sf t}_1 v_i \|\boldsymbol{v}\|_2\sqrt{T},
\end{equation*}
which bounds the total magnitude of the noise.

\noindent \textbf{Lower bounding $\sum_{i \in [N]} r_i \delta_{i, T}^{\sf shift}$:}

 \noindent We start by applying the tower rule and $\mathbb{E}[|\mathcal{E}_\ell||S_\ell] = V(S_\ell) + 1$,
\begin{align*}
    \mathbb{E}\left[\sum_{i \in [N]} r_i \delta_{i, T}^{\sf shift}\right] &= \mathbb{E}\left[\sum_{i \in [N]}\sum_{\ell \geq 1}r_i (V(S_\ell) + 1)(\pi(i, S_\ell) - \pi^{\sf UCB}(i, S_\ell)\right] \\ 
    &= \mathbb{E}\left[\sum_{S \in \mathcal{S}}\sum_{\ell \geq 1} \sum_{i \in [N]}  r_i y_\ell(S)(V(S) + 1)(\pi(i, S) - \pi^{\sf UCB}(i, S)\right],
    % &= \mathbb{E}\left[\sum_{S \in \mathcal{S}}\sum_{\ell \geq 1} \sum_{i \in [N]}\mathbb{E}\left[r_i y_\ell(S)(V(S) + 1)(\pi(i, S) - \pi^{\sf UCB}(i, S)\bigg|\mathcal{A}\right]\right] \\
    % &\geq -\mathbb{E}\left[\sum_{S \in \mathcal{S}}\sum_{\ell \geq 1} \sum_{i \in [N]}\mathbb{E}\left[r_i y_\ell(S)\left(\epsilon_{i, \ell} + \pi(i, S)\sum_{j \in S}\epsilon_{j, \ell}\right)\bigg|\mathcal{A}\right]\right]\\
    % &= -\mathbb{E}\left[\sum_{S \in \mathcal{S}, \ell \geq 1} y_\ell(S) \left(\sum_{i \in [N]} r_i \epsilon_{i, \ell} + \sum_{i \in [N]}r_i \pi(i, S)\sum_{j \in S}\epsilon_{j, \ell}\right)\right] 
\end{align*}
Conditionally on $\mathcal{A} \subset \mathcal{A}^{\sf CB1}$, we have:
\begin{align*}
    \sum_{S \in \mathcal{S}}\sum_{\ell \geq 1} \sum_{i \in [N]}  r_i y_\ell(S)(V(S) + 1)|\pi(i, S) - \pi^{\sf UCB}(i, S)| &\leq \sum_{S \in \mathcal{S}}\sum_{\ell \geq 1} \sum_{i \in [N]}  r_i y_\ell(S)\left(\epsilon_{i, \ell - 1} + \pi(i, S)\sum_{j \in S} \epsilon_{j, \ell - 1}\right) \\
    &= \sum_{\ell \geq 1}\sum_{S \in \mathcal{S}}y_\ell(S) \left(\sum_{i \in [N]} r_i \epsilon_{i, \ell - 1} + r(S)\epsilon_{\ell - 1}(S)\right) \\
    &= \sum_{\ell \geq 1}\sum_{i \in [N]} r_i \epsilon_{i, \ell - 1} + \sum_{\ell \geq 1}\sum_{S \in \mathcal{S}}y_\ell(S) r(S) \epsilon_{\ell - 1}(S) \\
    &\leq \sum_{\ell \geq 1}\sum_{i \in [N]} r_i \epsilon_{i, \ell - 1} + r_{\sf opt} \sum_{i \in [N]}\sum_{\ell \geq 1} \epsilon_{i, \ell - 1},
\end{align*}
\noindent where the first inequality stems from Lemma \ref{lem:lipshitz}, and the second inequality stems from $\epsilon_{\ell - 1}(S) \leq \sum_{\ell \geq 1}\epsilon_{i, \ell - 1}$ and $\sum_{S \in \mathcal{S}} y_\ell(S)r(S) \leq r_{\sf opt}$. By using Lemma \ref{lemma:errors}, we have conditionally on $\mathcal{A} \subset \mathcal{A}^{\sf CB2}$, by Cauchy-Schwartz inequality:
\begin{align*}
     \sum_{\ell \geq 1}\sum_{i \in [N]} r_i \epsilon_{i, \ell - 1} &\leq 2\Tilde{C}_1\sqrt{\log\left(\sqrt{N}T^2 + 1\right)} \|\boldsymbol{v}\cdot\boldsymbol{r}\|_2 \sqrt{T} + (\Tilde{C}_2 + \Tilde{C}_3)\|\boldsymbol{r}\|_1, \\
     \sum_{\ell \geq 1}\sum_{i \in [N]}\epsilon_{i, \ell - 1} &\leq 2\Tilde{C}_1\sqrt{\log\left(\sqrt{N}T^2 + 1\right)} \|\boldsymbol{v}\|_2 \sqrt{T} + (\Tilde{C}_2 + \Tilde{C}_3)N.
\end{align*}
where $\boldsymbol{v}\cdot\boldsymbol{r} = (r_i v_i)_{i \in [N]}$. Therefore, conditionally on $\mathcal{A}$
\begin{equation*}
   \sum_{i \in [N]} r_i \delta_{i, T}^{\sf shift} \geq -2\Tilde{C}_1\sqrt{\log\left(\sqrt{N}T^2 + 1\right)}\left(\|\boldsymbol{v}\|_2 r_{\sf opt} + \|\boldsymbol{v}\cdot\boldsymbol{r}\|_2\right) \sqrt{T} - (\Tilde{C}_2 + \Tilde{C}_3)(N r_{\sf opt} + \|\boldsymbol{r}\|_1),
\end{equation*}
which provides the desired lower bound.

\noindent \textbf{Deriving the exact regret bound:}

We have:
\begin{align*}
        \mathbb{E}\left[\sum_{\ell \geq 1} |\mathcal{E}_\ell| \sum_{i \in [N]}r_i \sum_{S \in \mathcal{S}} y_\ell(S) \pi^{\sf UCB}_{\ell-1}(i,S)\bigg| \mathcal{A}\right]\mathbb{P}(\mathcal{A}) &\geq \sum_{\ell \geq 1}|\mathcal{E}_\ell|\left((r_{\sf opt} - r_{\sf disc}\right)\left(1 - \frac{15}{T}\right) \\
        &= T r_{\sf opt} - T r_{\sf disc} - 15(r_{\sf opt} -  r_{\sf disc}) \\
        &\geq T r_{\sf opt} - r_{\sf inv} \sqrt{T} - 15r_{\sf opt}. \\
        \mathbb{E}\left[\sum_{i \in [N]} r_i \delta_{i, T}^{\sf shift}\bigg|\mathcal{A}\right]\mathbb{P}(\mathcal{A}) &\geq -2\Tilde{C}_1\sqrt{\log\left(\sqrt{N}T^2 + 1\right)}\left(\|\boldsymbol{v}\|_2 r_{\sf opt}+ \|\boldsymbol{v}\cdot\boldsymbol{r}\|_2\right) \sqrt{T} \\ &- (\Tilde{C}_2 + \Tilde{C}_3)(N r_{\sf opt} + \|\boldsymbol{r}\|_1) \\
        \mathbb{E}\left[\sum_{\ell \geq 1} |\mathcal{E}_\ell| \sum_{i \in [N]}r_i \sum_{S \in \mathcal{S}} y_\ell(S) \pi^{\sf UCB}_{\ell-1}(i,S)\bigg|\mathcal{A}^{\sf c}\right]\mathbb{P}(\mathcal{A}^{\sf c}) &\geq  0\\ \mathbb{E}\left[\sum_{i \in [N]} r_i \delta_{i, T}^{\sf shift}\bigg|\mathcal{A}^c\right]\mathbb{P}(\mathcal{A}^c) &\geq -N T r_{\max}.
\end{align*}

Therefore, 
\begin{align*}
    \text{Regret}_T &= T r_{\sf opt} - \mathbb{E}\left[\sum_{t \geq 1} r_{c_t} \right] \\
    &\leq  T r_{\sf opt} - \left(T r_{\sf opt} - A_q r_{\sf inv} \sqrt{T} - 15r_{\sf opt}\right) \\
    &+2\Tilde{C}_1\sqrt{\log\left(\sqrt{N}T^2 + 1\right)}\left(\|\boldsymbol{v}\|_2 r_{\sf opt}+ \|\boldsymbol{v}\cdot\boldsymbol{r}\|_2\right) \sqrt{T} \\ &+(\Tilde{C}_2 + \Tilde{C}_3)(N r_{\sf opt} + \|\boldsymbol{r}\|_1) \\
    &+ N T r_{\max} \\
    &= A_q r_{\sf inv} \sqrt{T} + 15r_{\sf opt} \\ &+ 2\Tilde{C}_1\sqrt{\log\left(\sqrt{N}T^2 + 1\right)}\left(\|\boldsymbol{v}\|_2 r_{\sf opt} + \|\boldsymbol{v}\cdot\boldsymbol{r}\|_2\right) \sqrt{T} \\ &+ (\Tilde{C}_2 + \Tilde{C}_3)(N r_{\sf opt} + \|\boldsymbol{r}\|_1) + 15 N r_{\max} \\
    &= B_0 \left(r_{\max}+ r_{\sf opt} + \frac{\|\boldsymbol{r}\|_1}{N} \right) N +  B_1 \left(r_{\sf inv} + \|\boldsymbol{v}\|_2 r_{\sf opt}  + \|\boldsymbol{v}\cdot\boldsymbol{r}\|_2\right)\sqrt{T} \\
    &\lesssim r_{\max} N + \left(r_{\sf inv} + \|\boldsymbol{v}\|_2 r_{\sf opt}  + \|\boldsymbol{v}\cdot\boldsymbol{r}\|_2\right)\sqrt{T}.
\end{align*}

\noindent where we set $B_0 = \max\left(15, \Tilde{C}_2 + \Tilde{C}_3\right)$ and $B_1 = \max\left(1, 2\Tilde{C}_1\sqrt{\log\left(\sqrt{N}T^2 + 1, A_q\right)}\right)$. This completes the proof of of Theorem \ref{main_theorem}.
\end{APPENDIX}
%
%   or 
%
% \begin{APPENDICES}
% \section{<Title of Section A>}
% \section{<Title of Section B>}
% etc
% \end{APPENDICES}

% References here (outcomment the appropriate case) 

% CASE 1: BiBTeX used to constantly update the references 
%   (while the paper is being written).
%\bibliographystyle{informs2014} % outcomment this and next line in Case 1
%\bibliography{<your bib file(s)>} % if more than one, comma separated

% CASE 2: BiBTeX used to generate mypaper.bbl (to be further fine tuned)
%\input{mypaper.bbl} % outcomment this line in Case 2

\end{document}